\def\toc{0}

\documentclass[11pt]{article}

\usepackage{amsmath, amssymb, amsthm}
\usepackage{enumerate}
\usepackage{framed}
\usepackage{verbatim}
\usepackage[margin=1.1in]{geometry}
\usepackage{microtype}
\usepackage{kpfonts}
\usepackage{paralist}
\usepackage{palatino}
	\DeclareMathAlphabet{\mathtt}{OT1}{cmtt}{m}{n}
	\SetMathAlphabet{\mathtt}{bold}{OT1}{cmtt}{bx}{n}
	\DeclareMathAlphabet{\mathsf}{OT1}{cmss}{m}{n}
	\SetMathAlphabet{\mathsf}{bold}{OT1}{cmss}{bx}{n}

\usepackage[usenames,dvipsnames]{xcolor}

\definecolor{DarkGreen}{rgb}{0.15,0.5,0.15}
\definecolor{DarkRed}{rgb}{0.6,0.2,0.2}
\definecolor{DarkBlue}{rgb}{0.2,0.2,0.6}
\definecolor{DarkPurple}{rgb}{0.4,0.2,0.4}
\usepackage[numbers]{natbib}
\usepackage{url}
\usepackage{hyperref}
\hypersetup{
	linktocpage=true,
	colorlinks=true,
	linkcolor=DarkBlue,
	citecolor=DarkBlue,
	urlcolor=DarkBlue,
}
\usepackage{indentfirst}
\usepackage[boxruled,vlined]{algorithm2e}
	\SetKwFor{While}{While}{}{}
	\SetKwFor{For}{For}{}{}
	\SetKwIF{If}{ElseIf}{Else}{If}{}{elif}{Else}{}
	\SetKw{Return}{Return}

\newcommand{\etal}{\emph{et al.}}

\newcommand{\ex}[2]{{\ifx&#1& \mathbb{E} \else \underset{#1}{\mathbb{E}} \fi \left[#2\right]}}
\newcommand{\pr}[2]{{\ifx&#1& \mathbb{P} \else \underset{#1}{\mathbb{P}} \fi \left[#2\right]}}
\newcommand{\var}[2]{{\ifx&#1& \mathrm{Var} \else \underset{#1}{\mathrm{Var}} \fi \left[#2\right]}}

\newcommand{\from}{:}

\newcommand{\zo}{\{0,1\}}
\newcommand{\pmo}{\{\pm1\}}

\newcommand{\set}[1]{\left\{#1\right\}}

\newcommand{\eps}{{\varepsilon}}

\newcommand{\Ber}{\ensuremath{\mathsf{Ber}}}

\newcommand{\SD}{\mathit{SD}}

\newcommand{\N}{\mathbb{N}}

\newcommand{\cA}{\mathcal{A}}
\newcommand{\cB}{\mathcal{B}}

\newcommand{\cD}{\mathcal{D}}

\newcommand{\cM}{\mathcal{M}}

\newcommand{\cP}{\mathcal{P}}

\newcommand{\cX}{\mathcal{X}}
\newcommand{\cY}{\mathcal{Y}}

\newtheorem{thm}{Theorem}[section]
\newtheorem{theorem}[thm]{Theorem}
\newtheorem{lem}[thm]{Lemma}
\newtheorem{lemma}[thm]{Lemma}
\newtheorem{fact}[thm]{Fact}
\newtheorem{clm}[thm]{Claim}
\newtheorem{subclm}[thm]{Sub-Claim}
\newtheorem{claim}[thm]{Claim}
\newtheorem{rem}[thm]{Remark}
\newtheorem{coro}[thm]{Corollary}
\newtheorem{prop}[thm]{Proposition}

	\theoremstyle{definition}
\newtheorem{defn}[thm]{Definition}

\newcommand{\dist}{\cP}

\newcommand{\query}{q}

\newcommand{\sU}{\mathsf{U}}

\newcommand{\AQ}[3]{\ensuremath{\mathsf{AQ}_{#1}[#2 \leftrightharpoons #3]}}
\newcommand{\naq}{\ensuremath{\mathsf{NAQ}}}
\newcommand{\NAQ}[3]{\ensuremath{\mathsf{NAQ}_{#1}[#2 \leftrightharpoons #3]}}
\newcommand{\all}{\ensuremath{\mathsf{AQ}}}
\newcommand{\comp}{\ensuremath{\mathsf{CAQ}}}

\newcommand{\key}{\ensuremath{m}}
\newcommand{\prg}{\ensuremath{G}}
\newcommand{\seed}{\ensuremath{s}}

\newcommand{\trunc}[1]{{trunc}_{#1}}

\title{The Limits of Post-Selection Generalization}

\author{
\makebox[1.5in]{\hfill Kobbi Nissim\thanks{Department of Computer Science, Georgetown University. Supported by NSF award CNS-1565387. \href{mailto:kobbi.nissim@georgetown.edu}{\texttt{kobbi.nissim@georgetown.edu}}}\hfill}
\and \makebox[1.5in]{\hfill Adam Smith\thanks{Computer Science Department, Boston University.   Supported by NSF awards  IIS-1447700 and AF-1763665, a Google Faculty Award and a Sloan Foundation research award.  \href{mailto:ads22@bu.edu}{\texttt{ads22@bu.edu}}.
}\hfill}
\and \makebox[1.5in]{\hfill Thomas Steinke\thanks{Work done while the author was at IBM Research -- Almaden.  \href{mailto:posel@thomas-steinke.net}{\texttt{posel@thomas-steinke.net}}} \hfill}
\and \makebox[1.5in]{\hfill Uri Stemmer\thanks{Work done while the author was at the Weizmann Institute of Science, supported by a Koshland Postdoctoral Fellowship and by the Israel Science Foundation (grants 950/16 and 5219/17). \href{mailto:u@uri.co.il}{\texttt{u@uri.co.il}}}  \hfill}
\and \makebox[1.5in]{\hfill Jonathan Ullman\thanks{College of Computer and Information Science, Northeastern University.  Research supported by NSF CAREER award CCF-1750640, NSF award CCF-1718088, and a Google Faculty Award. \href{mailto:jullman@ccs.neu.edu}{\texttt{jullman@ccs.neu.edu}}}\hfill}
}
\date{}

\begin{document}
\pagenumbering{gobble}

\maketitle

\begin{abstract}
While statistics and machine learning offer numerous methods for ensuring generalization, these methods often fail in the presence of \emph{adaptivity}---the common practice in which the choice of analysis depends on previous interactions with the same dataset.  A recent line of work has introduced powerful, general purpose algorithms that ensure \emph{post hoc generalization} (also called \emph{robust} or \emph{post-selection} generalization), which says that, given the output of the algorithm, it is hard to find any statistic for which the data differs significantly from the population it came from.

In this work we show several limitations on the power of algorithms satisfying post hoc generalization.  First, we show a tight lower bound on the error of any algorithm that satisfies post hoc generalization and answers adaptively chosen statistical queries, showing a strong barrier to progress in post selection data analysis.  Second, we show that post hoc generalization is not closed under composition, despite many examples of such algorithms exhibiting strong composition properties.
\end{abstract}

\ifnum\toc=1
\newpage

\tableofcontents
\vfill
\newpage
\fi

\pagenumbering{arabic}
\section{Introduction}

Consider a dataset $X$ consisting of $n$ independent samples from some  unknown population $\dist$.  How can we ensure that the conclusions drawn from $X$ \emph{generalize} to the population $\dist$?  Despite decades of research in statistics and machine learning on methods for ensuring generalization, there is an increased recognition that many scientific findings generalize poorly (e.g. \citet{Ioannidis05}).  While there are many reasons a conclusion might fail to generalize, one that is receiving increasing attention is \emph{adaptivity}---when the choice of method for analyzing the dataset depends on previous interactions with the same dataset. 

 Adaptivity can arise from many common practices, such as exploratory data analysis, using the same data set for feature selection and regression, and the re-use of datasets across research projects.  Unfortunately, adaptivity invalidates traditional methods for ensuring generalization and statistical validity, which assume that the method is selected independently of the data. 
The misinterpretation of adaptively selected results has even been blamed for a ``statistical crisis'' in empirical science~\cite{GelmanL13}.

Once methods are selected adaptively, data analysts must take selection into account when interpreting results---a problem the statistics literature refers to as \emph{selective inference}, or \emph{post-selection inference}. 
Several approaches have been devised for statistical inference after selection.
These generally require conditional inference algorithms, or uniform convergence arguments, tailored to a particular sequence of analyses. Perhaps more fundamentally, this line of work is not prescriptive, in that it does not provide design principles that guide the selection of a sequence of analyses to improve later inference. 
We discuss that work further in Section~\ref{sec:relwork}.

A recent line of work initiated by Dwork \etal~\cite{DworkFHPRR15} posed the question: Can we design \emph{general-purpose} algorithms for ensuring generalization in the presence of adaptivity, together with guarantees on their accuracy?  These works identified properties of an algorithm that ensure good generalization of queries selected based on its output, including \emph{differential privacy}~\citep{DworkFHPRR15,BassilyNSSSU16}, information-theoretic measures \citep{DworkFHPRR-nips-15, RussoZ15, RogersRST16,XR17}, and compression \citep{CummingsLNRW16}.  They also identified many powerful general-purpose algorithms satisfying these properties, leading to algorithms for post-selection data analysis with greater statistical power than all previously known approaches.  

Each of the aforementioned properties gives incomparable generalization guarantees, and allows for qualitatively different types of algorithms.  The common thread in each of these approaches is to establish a notion of \emph{post hoc generalization} (first named \emph{robust generalization} by  \citet{CummingsLNRW16}). Informally, an algorithm $\cM$ satisfies post hoc generalization if there is no way, given only the output of $\cM(X)$, to identify any \emph{statistical query}~\cite{Kearns93} such that the value of that query on the dataset is significantly different from its answer on the whole population.\footnote{The definition extends seamlessly to richer classes of statistics, but we specialize to these queries for concreteness.  That our negative results hold for such a simple class of queries only makes our results stronger.}

More formally: a \emph{statistical query} (or \emph{linear functional}) is defined by a function $q:\cX\to [-1,1]$, where $\cX$ is the set of possible data points. The query's \emph{empirical mean} on a data set $X$ is its average over the points in $X$, i.e.,  $q(X) = \frac1n \sum_i q(X_i)$. Given a distribution $\cP$ (called the population) on $\cX$, the query's \emph{population mean} is $q(\cP) = \ex{x\sim \cP}{q(x)}$, the expectation of $q$ on a fresh sample from $\cP$. 

\begin{defn}[Post Hoc Generalization~\cite{CummingsLNRW16}]\label{defn:phg}
An algorithm $\cM \from \cX^n \to \cY$ satisfies \emph{$(\eps,\delta)$-post hoc generalization} if for every distribution $\cP$ over $\cX$ and every algorithm $\cA$ that outputs a bounded function $q \from \cX \to [-1,1]$, if $X \sim \cP^{\otimes n}, Y \sim \cM(X),$ and $\query \sim \cA(Y)$, then
$$
\pr{}{\left| q(\cP) - q(X) \right| > \eps} \leq \delta,
$$
where $q(\cP) = \ex{x \sim \cP}{q(x)}$ and $q(X) = \frac1n \sum_i q(X_i)$, and the probability is over the sampling of $X$ and any randomness of $\cM$ and $\cA$.
\end{defn}

We phrase the definition as a tail bound on the post hoc
generalization error, 
but one could also cast the definition in terms of the generalization
error's moments (e.g.~mean squared error). For the
purpose of proving asymptotic lower bounds, the definition we use is
more general.

Post hoc generalization is easily satisfied whenever $n$ is large
enough to ensure \emph{uniform convergence} for the class of
statistical queries.  However, uniform convergence is only satisfied
in the unrealistic regime where $n$ is much larger than $|\cX|$.
Algorithms that satisfy post hoc generalization are interesting in the
realistic regime where there will \emph{exist} queries $q$ for which
$q(\cP)$ and $q(X)$ are far, but these queries cannot be \emph{found}.

Since all existing general-purpose algorithms for post-selection data analysis are analyzed via post hoc generalization, it is crucial to understand what we can achieve with algorithms satisfying post hoc generalization.  This work shows two types of strong limitations on post hoc generalization:
\begin{itemize}
\vspace{-3pt}
\item We show new, almost tight limits on the accuracy of any post hoc generalizing algorithm for estimating the answers to adaptively chosen statistical queries.
\vspace{-1pt}
\item We show that the composition of algorithms satisfying post hoc generalization does not always satisfy post hoc generalization.
\end{itemize}
\vspace{-3pt}
Our results identify natural barriers to progress in this area, and highlight important challenges for future research on post-selection data analysis.

\subsection{Our Results}

\subsubsection{Sample Complexity Bounds for Statistical Queries} 

Our first contribution is a strong new lower bound on any algorithm that answers a sequence of adaptively chosen statistical queries while satisfying post hoc generalization. State-of-the-art algorithms for this problem do satisfy post hoc generalization, thus our lower bound shows that improving on the state of the art would require a fundamentally different analysis paradigm.

\citet{DworkFHPRR15} were the first to study the sample size required to answer adaptive statistical queries; a recent line of work deepened the study of their model \citep{BassilyNSSSU16,HardtU14,SteinkeU14,CummingsLNRW16,RogersRST16,FeldmanS17,FishRR17}.
In the model of~\citet{DworkFHPRR15}, there is an underlying distribution $\cP$ on a set $\cX$ that an analyst would like to study.  An algorithm $\cM$ holds a sample $X \sim \cP^{\otimes n}$, receives statistical queries $q$, and returns accurate answers $a$ such that $a \approx q(\cP)$.
To model adaptivity, consider a \emph{data analyst} $\cA$ that issues a sequence of queries $q^1,\dots,q^k$ where each query $q^j$ may depend on the answers $a^{1},\dots,a^{j-1}$ given by the algorithm in response to previous queries.  
Our goal is to design mechanisms $\cM$ that answer a sequence of $k$ queries selected by an arbitrary (possibly adversarial) analyst at a desired accuracy $\eps$ using as small a sample size $n$ as possible.

The simplest algorithm $\cM$ would return the empirical mean $q^j(X) = \frac1n \sum_i q^j(X_i)$ in response to each query. One can show that this algorithm answers each query to within $\pm \eps$ if $n \geq \tilde{O}(k/\eps^2)$.\footnote{Formally, in order to rule out pathological queries, this sample-complexity guarantee applies only if we round the empirical mean to some suitable precision $\pm 1/\mathrm{poly}(n)$.}  Surprisingly, we can improve the sample complexity to $n \geq \tilde{O}(\sqrt{k}/\eps^2)$---a quadratic improvement in $k$---by returning $q(X)$ perturbed with carefully calibrated noise~\cite{DworkFHPRR15,BassilyNSSSU16}.  The analysis of this approach relies on post hoc generalization: one can show that no matter how the analyst selects queries, each query will satisfy $q(\cP) \approx q(X)$ with high probability. The noise distribution also ensures that the difference between each answer $a$ and the empirical mean $q(X)$ is small, hence $a\approx q(\cP)$.

Our main result shows that the sample complexity $n = \tilde{O}(\sqrt{k}/\eps^2)$ is essentially optimal for \emph{any} algorithm that uses the framework of post hoc generalization.  Our construction refines the techniques in~\citet{HardtU14} and \citet{SteinkeU14}---which yield a lower bound of $n = \Omega(\sqrt{k})$ for $\eps = 1/3$.

\begin{theorem}[Informal] \label{thm:mainlb}
If $\cM$ takes a sample of size $n$, satisfies $(\eps,\delta)$-post hoc generalization, and for every distribution $\cP$ over $\cX = \pmo^{k + O(\log(n/\eps))}$ and every data analyst $\cA$ who asks $k$ statistical queries,
$$
\pr{}{ \exists j \in [k],~|q^j(\cP) - a^j| > \eps}\leq\delta
$$ (where the probability is taken over $X \sim \cP^{\otimes n}$ and
the coins of $\cM$ and  $\cA$), 
then $n = \tilde{\Omega}(\sqrt{k}/\eps^2)$.\footnote{Independently, \citet{WangThesis} proved a quantitatively similar bound to Theorem~\ref{thm:mainlb}.  However, Wang's  lower bound assumes that the algorithm $\cM$ can see only the empirical mean $q(X)$ of each query, and not the raw data $X$. Their bound also applies for a slightly different (though closely related) class of statistics, where all possible query values are jointly Gaussian.
}
\end{theorem}

The dimensionality of $\cX$ required in our result is at least as large as $k$; that dependence is essentially necessary.  Indeed, if the support of the distribution is $\pmo^{d}$, then there is an algorithm $\cM$ that takes a sample of size just $\tilde{O}(\sqrt{d} \log(k)/\eps^3)$~\cite{DworkFHPRR15, BassilyNSSSU16}, so the conclusion is simply false if $d \ll k$.
Even when $d \ll k$, the aforementioned algorithms require running time at least $2^d$ per query.  \citet{HardtU14} and \citet{SteinkeU14} also showed that any \emph{polynomial time} algorithm that answers $k$ queries to constant error requires $n = \Omega(\sqrt{k})$.  We improve this result to have the optimal dependence on $\eps$.

\begin{thm} [Informal] \label{thm:mainlbcomp}
Assume pseudorandom generators exist and let $c > 0$ be any constant.  If $\cM$ takes a sample of size $n$, has polynomial running time, satisfies $(\eps,\delta)$-post hoc generalization, and for every distribution $\cP$ over $\cX = \pmo^{(n+k)^{c} + O(\log(n/\eps))}$ and every data analyst $\cA$ who asks $k$ statistical queries, $$\pr{}{ \exists j \in [k],~|q^j(\cP) - a^j| > \eps}\leq\delta,$$ then $n = \tilde{\Omega}(\sqrt{k}/\eps^2)$, where the probability is taken over $X \sim \cP^{\otimes n}$ and the coins of $\cM$ and  $\cA$.
\end{thm}

\subsubsection{Negative Results for Composition}

One of the motivations for studying post hoc generalization is to allow for exploratory data analysis and dataset re-use.  In these settings, the same dataset may be analyzed by many different algorithms, each satisfying post hoc generalization.  Thus it is important to understand whether the \emph{composition} of these algorithms also satisfies post hoc generalization.  We show that this is not always the case.
\begin{theorem}\label{thm:mainNoComposition}
For every $n\in\N$ there is a collection of $\ell=O(\log n)$ algorithms $\cM_1,\dots,\cM_\ell$ that take $n$ samples from a distribution over $\cX = \zo^{O(\log n)}$ such that 
\begin{enumerate}
\item for every $\delta > 0$ and for $\eps=O(\sqrt{\frac{\log(n/\delta)}{n^{.999}}})$, each of these algorithms satisfies $(\eps,\delta)$-post hoc generalization, but 
\item the composition $(\cM_1,\dots,\cM_\ell)$ does not satisfy $(1.999,.999)$-post hoc generalization.
\end{enumerate}
\end{theorem}

By standard anti-concentration results, no algorithm satisfies $(\eps,\delta)$-post hoc generalization for $\eps = o(\sqrt{\log(1/\delta)/n})$.  On the other hand, every algorithm trivially satisfies $(2,0)$- and $(0,1)$-post-hoc generalization.  Thus, Theorem~\ref{thm:mainNoComposition} states that there is a set of $O(\log n)$ algorithms that have almost optimal post hoc generalization, but whose composition does not have any non-trivial post hoc generalization.  Certain subclasses of post hoc generalizing algorithms---such as \emph{differential privacy}~\cite{DworkMNS06}---do satisfy composition where the parameters $\eps,\delta$ grow at worst linearly in $\ell$.  In contrast, for arbitrary post hoc generalization, $\eps,\delta$ can grow with $2^{\Omega(\ell)}$.  Our results give additional motivation to studying subclasses that do compose.

We also consider a relaxed notion of \emph{computational post hoc generalization}, for which we show that composition can fail even for just two algorithms.  Informally, computational post hoc generalization requires that Definition~\ref{defn:phg} hold when the analyst $\cA$ runs in polynomial time.
\begin{theorem}\label{thm:mainNoCompositionComp}
Assume pseudorandom generators exist.  For every $n\in\N$ there are two algorithms $\cM_1, \cM_2$ that take $n$ samples from a distribution over $\cX = \zo^{O(\log n)}$ such that
\begin{enumerate}
\item for every $\delta > n^{-O(1)}$ and for $\eps=O(\sqrt{\frac{\log(n/\delta)}{n^{.999}}})$, both algorithms satisfy $(\eps,\delta)$-computational post hoc generalization, but 
\item the composition $(\cM_1,\cM_2)$ is not $(1.999, .999)$-computationally post hoc generalizing.
\end{enumerate}
\end{theorem}

We prove the information-theoretic result (Theorem~\ref{thm:mainNoComposition}) in Section~\ref{sec:comp}.  The computational result (Theorem~\ref{thm:mainNoCompositionComp}) is proved in Section~\ref{sec:compcomp}.

\subsection{Related Work}
\label{sec:relwork}

In addition to the upper and lower bounds mentioned above for adaptive linear queries, a number of works have explored variants of the model, notably the models of jointly Gaussian queries \citep{RussoZ15,WangLF16,WangThesis} and a Bayesian model with symmetric information \citep{Elder16,Elder17}.

In the statistics literature, work on selective inference dates at least to the works of \citet{Freedman83},~\citet{HurvichT90}, and \citet{Potscher91}. The last decade or so has seen a resurgence of interest, for example in \cite{leeb2005model,leeb2006can,Efron14,FithianST14,TaylorT15,Berk+13,Buja+15} (this list is necessarily incomplete). One line of work due to \citet{Berk+13} and~\citet{Buja+15} provides uniform validity results for all analyses in particular classes. As we mentioned above, uniform convergence is not possible in many settings without prohibitively large sample sizes. Another line of work looks at selective inference by explicitly conditioning on the sequence of previous query answers (see  \citet{BiMXT-inferactive17} for a high-level summary of the approach). Explicit conditioning has the advantage of optimality in several contexts, but  requires formulating a prior over possible distributions, and can be computationally infeasible. Perhaps most problematically, the work we are aware of along these lines is more ``descriptive'' than ``prescriptive,'' providing no estimates of power or accuracy before the execution of an experiment, and thus no guidance on the optimal design of the experiment. Understanding the connections between that line of work and the work on which we build directly remains an intriguing open problem.

While our work considers biases arising from adaptive \emph{analysis} of iid data, a recent work of \citet{zou17} investigates the bias introduced by adaptive \emph{sampling} of data.  These questions have similar motivation, but are technically orthogonal.

\section{Lower Bounds for Statistical Queries} \label{sec:sqs}

\vspace{-5pt}

\subsection{Post Hoc Generalization for Adaptive Statistical Queries}

We are interested in the ability of \emph{interactive} algorithms satisfying post hoc generalization to answer a sequence of statistical queries.  Definition~\ref{defn:phg} applies to such algorithms via the following experiment.
\begin{algorithm}[ht]
\caption{$\AQ{\cX,n,k}{\cM}{\cA}$\label{alg:aq}}
$\cA$ chooses a distribution $\cP$ over $\cX$\\
$X \sim \cP^{\otimes n}$ and $X$ is given to $\cM$ (but not to $\cA$)\\
\For{$j = 1,\dots,k$}{
$\cA$ outputs a statistical query $q^j$ (possibly depending on $q^1, a^1,\dots,q^{j-1},a^{j-1}$) \\
$\cM(X)$ outputs $a^{j}$
}
\end{algorithm}

\begin{defn} An algorithm $\cM$ is \emph{$(\eps,\delta)$-post hoc generalizing for $k$ adaptive queries over $\cX$ given $n$ samples} if for every adversary $\cA$,
$$
\pr{\AQ{\cX,n,k}{\cM}{\cA}}{\exists j \in [k]~~\left| q^j(X) - q^j(\cP) \right| > \eps } \leq \delta.
$$
\end{defn}

\begin{defn}
An algorithm $\cM$ is \emph{$(\eps,\delta)$-accurate for $k$ adaptive queries over $\cX$ given $n$ samples} if for every adversary $\cA$,
$$
\pr{\AQ{\cX,n,k}{\cM}{\cA}}{\exists j \in [k]~~\left| a^j - q^j(\cP) \right| > \eps } \leq \delta.
$$
\end{defn}

\subsection{A Lower Bound for Natural Algorithms}

We begin with an information-theoretic lower bound for a class of algorithms $\cM$ that we call \emph{natural algorithms.}  These are algorithms that can only evaluate the query on the sample points they are given.  That is, an algorithm $\cM$ is \emph{natural} if, when given a sample $X = (X_1,\dots,X_n)$ and a statistical query $q: \cX \to [-1,1]$, the algorithm $\cM$ returns an answer $a$ that is a function only of $(q(X_1),\dots,q(X_n))$. In particular, it cannot evaluate $q$ on data points of its choice. Many algorithms in the literature have this property.  Formally, we define natural algorithms via the game $\NAQ{\cX,n,k}{\cM}{\cA}$.  This game is identical to $\AQ{\cX,n,k}{\cM}{\cA}$ except that when $\cA$ outputs $q^j$, $\cM$ does not receive all of $q^j$, but instead receives only $q^j_X = (q^j(X_1),\dots,q^j(X_n))$.

\begin{algorithm}[ht]
\caption{$\NAQ{\cX,n,k}{\cM}{\cA}$\label{alg:naq}}
$\cA$ chooses a distribution $\cP$ over $\cX$\\
$X \sim \cP^{\otimes n}$ and $X$ is given to $\cM$ (but not to $\cA$)\\
\For{$j = 1,\dots,k$}{
$\cA$ outputs a statistical query $q^j : \cX \to [-1,1]$ (possibly depending on $q^1, a^1,\dots,q^{j-1},a^{j-1}$) \\
$\cM$ receives $q^j_X = (q^j(X_1),\dots, q^j(X_n))$ \\
$\cM$ outputs $a^{j}$ (possibly depending on $q^1_X,a^1,\dots,q^{j-1}_{X},a^{j-1},q^{j}_{X}$)
}
\end{algorithm}

\begin{thm}[Lower Bound for Natural Algorithms] \label{thm:naturallb}
There is an adversary $\cA_{\naq}$ and there are absolute constants $c_0, C > 0$ such that the following holds for every natural algorithm $\cM$ answering $k$ queries using $n$ samples, and for every $\sqrt{\frac{C\log n}{n}} \leq \eps \leq c_0$.  Let $N = 8n/\eps$ denote the universe size.  If
$$\pr{\NAQ{[N],n,k}{\cM}{\cA_{\naq}}}{\exists j \in [k]~~\left| q^j(X) - q^j(\cP) \right| > \eps \bigvee \left| a^j - q^j(\cP)\right| > \eps } \leq \tfrac{1}{100}$$
then $n = \Omega\left(\frac{\sqrt{k}}{\eps^2\sqrt{\log(1/\eps)}}\right)$.
\end{thm}

The proof uses the analyst $\cA_{\naq}$ described in
Algorithm~\ref{alg:tracer}. Throughout, we refer to the items in the sets $A^1 \subseteq A^2 \subseteq \dots \subseteq A^{k+1}$ that the adversary maintains as \emph{accused} (of belonging to the sample). For notational convenience, $\cA_{\naq}$ actually asks $k+1$ queries, but this does not affect the final result.

To simplify the analysis, throughout this section and Section~\ref{sec:lifting} we assume that the sample $X$ is drawn {\em without replacement} from $\cP$ (e.g., using rejection sampling), and prove Theorem~\ref{thm:naturallb} and its consequences for this variant of the game. This is without loss of generality because, in our constructions, the target distribution will be uniform over a subset of the domain where the size of the subset is known to the mechanism. Under these conditions, a sample with replacement can easily be simulated from a sample without replacement.\footnote{One also needs to account for the difference between the empirical means of the simulated sample and of the given one; since $N = 8n/\eps$, this difference is at most $\eps/2$ except with probability $e^{-\Omega(\eps n)}$, which amounts to changing $\eps$ by a factor of $2$ and we ignore it for simplicity.} Thus, from now until the end of Section~\ref{sec:lifting}, we treat $X$ as a {\em set} instead of a {\em multiset}.

\begin{algorithm}[ht]\label{alg:tracer}
\caption{$\cA_{\naq}$}
\textbf{Parameters:} sample size $n$, universe size $N = \frac{8n}{\eps}$, number of queries $k$, target accuracy $\eps$ \\\vspace{3pt}

Let $\cP \gets \sU_{[N]}$, $A^1 \gets \emptyset$, and $\tau \gets 30\eps \sqrt{2k \log(\frac{96}{\eps})} + 10\eps$ \\
\For{$j \in [k]$}{
Sample $p^j \sim \sU_{[0,1]}$ \\
\For{$i \in [N]$}{
Sample $\tilde q^j_i \sim \Ber(p^j)$ and let $q^j(i) \gets
\left\{ \begin{array}{cl} \tilde q^j_i & i \notin A^j \\ 0 & i \in
                                                          A^j \end{array}
                                                      \right.$
}
Ask query $q^j$ and receive answer $a^j$ \\
\For{$i \in [N]$}{
Let $z^j_i \gets \left\{ \begin{array}{cl} \trunc{10\eps}(a^j -
                            p^j) \cdot (q^j_i - p^j) & i \notin A^j \\ 0
                                                     & i \in
                                                       A^j \end{array}
                                                   \right.$ \\
\quad where $\trunc{10\eps}(x)$ takes $x\in \mathbb{R}$ and returns the
nearest point in $[-10\eps,10\eps]$ to $x$.\\
Let $A^{j+1} \gets \set{i \in[N] : \ \left|\sum_{\ell=1}^{j}
    z_i^\ell\right| > \tau - 10\eps}$ (\emph{N.B.} By construction, $A^j \subseteq A^{j+1}$.)
}
}
\For{$i \in [N]$}{
Define $z_i \gets \sum_{j=1}^{k} z^j_i$ and $q^*_i \gets \frac{z_i}{\tau}$}
Let $q^* \from [N] \to [-1,1]$ be defined by $q^*(i) \gets q^*_i$
\end{algorithm}

We prove Theorem \ref{thm:naturallb} by proving the following dichotomy proposition: If the mechanism ``survives'' the initial $k$ queries $q^j$, then the final query $q^*$ gives significantly different answers on $X$ and $\cP$ (in which case $\cM$ does not satisfy post hoc generalization).  Formally,
\begin{prop} \label{prop:natural}
For an appropriate choice of $k = \Theta(\eps^4 n^2 \log(1/\eps))$ and $n, \frac{1}{\eps}$ sufficiently large with $\eps^2 n \geq C\log n$, for any natural $\cM$, with probability at least $2/3$, either
\begin{enumerate}
\item $\exists j \in [k]~~| a^j - q^j(\cP) | > \eps$ or $|q^j(X) - q^j(\cP)| > \eps$, or 
\item $q^*(X) - q^*(\cP) > \eps$
\end{enumerate}
where the probability is taken over the game $\NAQ{[N],n,k}{\cM}{\cA_{\naq}}$.
\end{prop}

We prove Proposition \ref{prop:natural} using  a series of claims. The first claim states that none of the values $z_i$ are ever too large in absolute value, which follows immediately from the definition of the set $A^j$ and the fact that each term $z^j_i$ is bounded.
\begin{clm} \label{clm:boundedz}
For every $i \in [N]$, $|z_i| \leq \tau$.
\end{clm}
\begin{proof}
Note that for every $i,j$, we have $|z_i^j| \leq 10\eps$.  Now, if $i \notin A^{k+1}$ then by definition $\bigl|\sum_{\ell=1}^{k} z^{\ell}_{i}\bigr|  \leq \tau-10\eps$, and the claim follows.  Otherwise, suppose $j+1$ is the minimal value such that $i \in A^{j+1}$, then 
$$|z_i| = \left| \sum_{\ell=1}^{j} z_i^\ell + \sum_{\ell=j+1}^{k} z_i^\ell \right| \leq \left| \sum_{\ell=1}^{j-1} z_i^\ell\right| + |z_i^j| + 0 \leq \tau-10\eps + 10\eps + 0 = \tau \qedhere$$
\end{proof}

The next claim states that the sum of the scores over all $i$ not in the sample is small.
\begin{clm} \label{clm:natural_innocent}
With probability at least $\frac{99}{100}$, $\sum_{i \in [N] \setminus X} z_i = O(\tau \sqrt{N}).$
\end{clm}
\vspace{-8pt}
\begin{proof} Fix a choice of $(p^1,\dots,p^k) \in[0,1]^{k}$, the in-sample query values $(q^1_{X},\dots,q^k_{X}) \in \zo^{n \times k}$, and the answers $(a^{1},\dots,a^{k})\in [0,1]^{k}$.  Conditioned on these, the out-of-sample value streams $\{\tilde q^j_i\}_{j \in [k]}$ for $i \notin X$ remain independent with $\tilde q^j_i \sim \Ber(p^j)$: everything we conditioned on is a function of the biases and the in-sample values alone, so, by the principle of deferred decisions, the out-of-sample values may be drawn after the interaction is complete.  Consequently, the scores $z_i$ for $i \notin X$ are independent and identically distributed.  Each is bounded by $\tau$ (Claim~\ref{clm:boundedz}), and each has expectation $0$: the increments $\tilde z^j_i := \trunc{10\eps}(a^j - p^j)(\tilde q_i^j - p^j)$ have conditional mean $\trunc{10\eps}(a^j - p^j) \cdot \ex{}{\tilde q^j_i - p^j} = 0$, and $z_i$ is a stopped sum of them (stopped at $i$'s first threshold crossing), so $\ex{}{z_i} = 0$ by optional stopping. 
By Hoeffding's inequality, with probability at least $\frac{99}{100}$
$
\sum_{i \in [N] \setminus X} z_i = O(\tau \sqrt{N})
$
as desired.  The claim now follows by averaging over all of the choices we fixed.
\end{proof}

The final claim states that if $\cM$ does not fail during the first $k$ rounds, then the sum of the scores over the sample points is large.

\begin{clm} \label{clm:natural_guilty} There exists $c > 0$ such that, 
for all sufficiently small $\eps$ and sufficiently large $n$ with $\eps^2 n \geq C\log n$, with probability at least $\frac{49}{50}$, either 
\begin{enumerate}
\item $\exists j \in [k]~:~|a^j - q^j(\cP)| > \eps$ or $|q^j(X) - q^j(\cP)| > \eps$ (large error or a post hoc generalization violation among the tracing queries), or 
\item $\sum\limits_{i \in X} z_i \geq ck$ (high scores in sample).
\end{enumerate}
\end{clm}

The proof of Claim~\ref{clm:natural_guilty} is non-trivial, and is deferred to Section~\ref{subsub:guilty}.  We first use Claims~\ref{clm:boundedz}, \ref{clm:natural_innocent} and~\ref{clm:natural_guilty} to prove the proposition, in order to keep the high level structure clear.  Informally, the argument for the proposition is that (by Claim~\ref{clm:natural_innocent}) the $z_i$ score of typical out-of-sample points is low, while (by Claim~\ref{clm:natural_guilty}) the $z_i$ score of typical in-sample points is high, and so the query $q^*(i) \propto z_i$ has a high empirical mean but a low population mean.

\begin{proof}[Proof of Proposition~\ref{prop:natural}]
Suppose that for every $j$ we have $|a^j - q^j(\cP)| \leq \eps$ and $|q^j(X) - q^j(\cP)| \leq \eps$.  By Claim~\ref{clm:natural_guilty}, with probability at least $49/50$ we may then assume that $\sum\limits_{i \in X} z_i \geq k/18$.

Then, we can show that, when $n$ is sufficiently large and $k \gtrsim \eps^4 n^2$, the final query $q^*$ will violate post hoc generalization.  The following calculation shows that for the query $q^*$ that we defined, $q^*(X) - q^*(\cP) > \eps$.
\begin{align*}
q^*(X) - q^*(\cP) ={} &\frac{1}{n} \sum_{i \in X} \frac{z_i}{\tau} - \frac{1}{N} \sum_{i \in [N]} \frac{z_i}{\tau} \\
={} &\left(\frac{1}{n\tau} - \frac{1}{N\tau}\right) \sum_{i \in X} z_i - \frac{1}{N\tau}\sum_{i \in [N] \setminus X} z_i \\
\geq{} &\frac{1}{2n\tau} \sum_{i \in X} z_i - \frac{1}{N\tau} \sum_{i \in [N] \setminus X} z_i \tag{$N = \frac{8n}{\eps} \geq 2n$}\\
={} &\frac{1}{2n\tau} \cdot \Omega(k) - \frac{1}{N\tau}\cdot O(\tau \sqrt{N}) \tag{Claims \ref{clm:natural_innocent} and \ref{clm:natural_guilty}} \\
={} &\Omega\left(\frac{k}{n\eps \sqrt{k\log(96/\eps)}}\right) - O\left(\frac{\tau\sqrt{N}}{N\tau}\right) \tag{$\tau = \Theta(\eps \sqrt{k\log(96/\eps)})$} \\
={} &\Omega\left(\frac{\sqrt{k/\log(96/\eps)}}{n \eps}\right) - O\left(\frac{1}{\sqrt{N}}\right) \\
\end{align*}

Now, choosing an appropriate $k = \Theta(\eps^4 n^2\log(1/\eps))$, we will have that $q^*(X) - q^*(\cP) > \eps$.  By this choice of $k$, the first term in the final line above will be at least $2\eps$.  Also, $1/\sqrt{N} = \sqrt{\eps/(8n)}$, so for sufficiently large $n$ the $O(1/\sqrt{N})$ term in the final line above is at most $\eps/4$, and altogether $q^*(X) - q^*(\cP) \geq 2\eps - \tfrac{\eps}{4} > \eps$.  Thus, by Claims~\ref{clm:natural_innocent} and~\ref{clm:natural_guilty}, either $\cM$ fails to be accurate, so that $\exists j \in [k]~|a^j - q^j(\cP)| > \eps$, or some tracing query itself violates post hoc generalization, so that $\exists j \in [k]~|q^j(X) - q^j(\cP)| > \eps$, or we find a query $q^*$ such that $q^*(X) - q^*(\cP) > \eps.$

Collecting the failure probabilities ($1/50$ for Claim~\ref{clm:natural_guilty} and $1/100$ for Claim~\ref{clm:natural_innocent}), the disjunction of the proposition holds with probability at least $1 - \tfrac{1}{50} - \tfrac{1}{100} \geq \tfrac{2}{3}$.
\end{proof}

\subsubsection{Proof of Claim~\ref{clm:natural_guilty}} \label{subsub:guilty}

In this section we prove Claim~\ref{clm:natural_guilty}, stating that, unless the mechanism fails during the first $k$ rounds, the sum of the scores over the in-sample points, $\sum_{i \in X} z_i$, is large with constant probability.  The proof proceeds in three steps.  First, we use a fingerprinting lemma to show that the expectation of (a related version of) this sum is large.  Second, we lift this expectation bound to a high-probability bound.  Third, we put the pieces together.

\paragraph{Step 1: The expected in-sample score of each round is large.}

This step relies on the following key lemma. The lemma has appeared in various forms~\cite{SteinkeU14,DworkSSUV15,SteinkeU17}; the form we use is \cite[Lemma 3.6]{BunSU17} (rescaled from $\{-1,+1\}$ to $\{0,1\}$).

\begin{lem}[Fingerprinting Lemma] \label{lem:fpc}
Let $f : \{0,1\}^m \to [0,1]$ be arbitrary. Sample $p \sim \sU_{[0,1]}$ and sample $x_1,\dots,x_m \sim \Ber(p)$ independently.  Then $$\ex{}{(f(x) - p) \cdot \sum_{i \in [m]} (x_i-p) +  \left| f(x)- \frac1m \sum_{i \in [m]} x_i \right|} \geq \frac{1}{12}.$$
\end{lem}

We would like to use the fingerprinting lemma to get an expectation bound on the per-round sum of scores:\footnote{Note that summing $z_i^j$ over all $i \in X$ is the
    same as summing over that round's unaccused elements $i\notin A^j$, since $z_i^j = 0$ for $i\in A^j$.}
$$\sum\limits_{i \in X}z^j_i=\trunc{10\eps}(a^j - p^j)\cdot\sum\limits_{i \in X \setminus A^j}(q^j_i - p^j),$$
    with the mechanism's answer $a^j$ playing the role of the function $f$ in the lemma, and the $q_i^j$'s and $p^j$ playing the roles of the $x_i$'s and $p$, respectively. The only mismatch is the truncation, which is a computation that depends on the bias $p^j$, not just the bits $q_i^j$, and is thus not captured by the fingerprinting lemma.
    
    Thus, we consider a variant of
  Algorithm~\ref{alg:tracer} that does not truncate the quantity
  $a^j-p^j$ to the range $\pm 10\eps$ when computing the score
  $z_i^j$ for each element $i$. Specifically, we consider scores based
  on the quantities
  $$\hat z_i^j =
  \begin{cases}
(a^j - p^j) \cdot (q_i^j - p^j) & \text{if }i \notin A^j\, ,\\
0 & \text{if } i \in A^j \,;
\end{cases}
\qquad\qquad \text{and}
  \quad \hat z_i = \sum_{j=1}^k \hat z_i^j\, .$$

We assume without loss of generality that the mechanism's answers to the first $k$ queries lie in $[0,1]$: projecting them onto $[0,1]$ can only improve both statistical and empirical accuracy.\footnote{Formally, this can be enforced by projecting each $a^j$ to $[0,1]$ in Algorithm~\ref{alg:tracer}.}  We prove the following claim.
\begin{clm}\label{clm:bigscores}
For every round $j$ and every history through round $j-1$, writing $\bar q^j := \tfrac{1}{N - |A^j|}\sum_{i \notin A^j} q^j_i$ for the average of the un-accused query values,
$$\ex{}{\;\sum\limits_{i \in X} \hat z^j_i \;\Big|\; \mathrm{history}} \;\geq\; \tfrac{1}{12} \;-\; \ex{}{\,\bigl|a^j - \bar q^j\bigr| \;\Big|\; \mathrm{history}}.$$
\end{clm}
\vspace{-8pt}
\begin{proof}
Fix a round $j$ and a history through round $j-1$.  
We start by analyzing $\sum\hat z_i^j$ over $i\in[N]$ instead of over $i\in X$; this will imply the result for the sum over $i\in X$, and is conceptually simpler. That is, we analyze:
    $$\sum_{i=1}^N \hat z_i^j  = \sum_{i \in [N]\setminus A^j}
    \hat z_i^j  = (a^j - p^j)\sum_{i \in [N]\setminus A^j} 
    (q^j_i - p^j) \, .$$
    We can now apply the Fingerprinting Lemma, with $m = N - |A^j|$, $p = p^j$, $x_i =
    \tilde q^j_i$ for $i \notin A^j$, and $f\left( (x_i)_{i\notin A^j}\right) = a^j$ (note that $f$ depends
    implicitly on $A^j$, but since we condition on the outcome of
    previous rounds, we may take $A^j$ as fixed for round $j$). We
    obtain
    $$\ex{}{\sum_{i=1}^N \hat z_i^j} \geq \frac 1 {12} - \ex{}{\left|
        a^j - \bar q^j \right|}.
    $$
    Next, the answer $a^j$ is a function of the in-sample values alone, so
    conditioning additionally on $p^j$ and on the in-sample values fixes
    $a^j - p^j$ while each out-of-sample factor $q^j_i - p^j$ ($i \notin X \cup
    A^j$) remains independent with mean zero; hence
    $\ex{}{(a^j - p^j)\sum_{i \notin X \cup A^j}(q^j_i - p^j)} = 0$, and the
    lower bound above holds verbatim for the in-sample sum
    $\sum\limits_{i \in X}\hat z^j_i$ alone.
\end{proof}

\paragraph{Step 2: From expectation to high probability.}
We now lift the expectation bound to a high-probability bound. The textbook approach for this would be to apply Azuma's inequality across the rounds, i.e., to show concentration for the sum 
$$\sum_{i \in X} \hat{z}_i=\sum_{j=1}^{k}
\underbrace{
\left(\sum_{i \in X} \hat{z}^j_i
\right)
}_{\text{\tiny per-round sum}}
=
\sum_{j=1}^{k}
\left(
(a^j - p^j)\sum_{i \in X \setminus A^j} 
    (q^j_i - p^j)
    \right).$$
That requires us to truncate these per-round sums. Such a truncation has two competing effects: (1) The truncation bound itself is accumulated by Azuma's inequality roughly $\sqrt{k}$ times across the $k$ rounds. This deviation should not overtake the expectation drift, so lower truncation is better. (2) The truncation undermines the conditional expectation bound of Claim~\ref{clm:bigscores}, possibly eating into it at every round, so higher truncation is better.  

When $\eps$ is not too small, truncating is straightforward: we can truncate each per-round sum by $\tilde{\Theta}(\sqrt{n})$, which holds w.h.p.\ regardless of $(a^j - p^j)$, i.e., regardless of the mechanism's behavior. This has virtually no effect on the expectation drift, keeping it $\Omega(k)$, and results in a deviation term of $\tilde{\Theta}(\sqrt{kn})$. For this deviation term to be dominated by the drift we need $k \gtrsim \tilde{\Theta}(n)$, which under our choice of $k$ amounts to $\eps \gtrsim \tilde{\Theta}(n^{-1/4})$.  For smaller $\eps$, however, we must truncate at the finer scale $\tilde{\Theta}(\eps\sqrt{n})$, and a truncation this tight is inactive only when $|a^j - p^j| = O(\eps)$. Thus, the naive way to approach this is to analyze (lower bound) the conditional expectation of
$$
\trunc{O(\eps)}(a^j - p^j)\cdot\sum_{i \in X \setminus A^j} 
    (q^j_i - p^j),
$$
which as we mentioned before cannot be done directly using the fingerprinting lemma. (This was the reason for introducing the untruncated $\hat{z}_i^j$'s!)

We resolve this by truncating the answers $a^j$ around the empirical average $q^j(X)$ instead of the bias $p^j$. (This happens only in the analysis; in Algorithm~\ref{alg:tracer} we still truncate around $p^j$.) This is helpful for two reasons: 
\begin{enumerate}
    \item Truncating around the empirical average is aligned with the requirements of the fingerprinting lemma (it is a function of the input bits and does not depend directly on the bias). So we can control the conditional expectations of the (empirically) truncated variables. In the next paragraph, instead of re-deriving this using the fingerprinting lemma, we analyze this using a reduction argument.
    \item As long as not too many sample points are accused, the empirical average tracks $p^j$ very closely, straying from it only with exponentially small probability. This allows us to connect the truncation around the empirical average to the actual truncation in the algorithm around the bias. One small wrinkle is that this requires us to enforce a stopping time for when too many sample points are accused, at which point post hoc generalization is violated (so the mechanism fails anyway). 
\end{enumerate}

\paragraph{Reduction to mechanisms that are empirically accurate with probability 1.}
Given a natural mechanism $\cM$, define the natural mechanism $\cM'$ as follows: at each round $j$, $\cM'$ runs $\cM$ (with the same internal randomness) to obtain a tentative answer $a^j$; if $|a^j - q^j(X)| \leq 3\eps$ it outputs $a^j$, and otherwise it outputs $q^j(X) := \frac1n\sum_{t=1}^n q^j(X_t)$.  Since $q^j(X)$ is computable from $\cM'$'s input, $\cM'$ is natural, and $|a'^j - q^j(X)| \leq 3\eps$ holds for every $j$ with probability $1$.

\begin{lem}[Patching]\label{lem:patch}
Suppose that Claim~\ref{clm:natural_guilty} holds for every natural mechanism satisfying $|a^j - q^j(X)| \leq 3\eps$ for all $j$ with probability $1$.  Then it holds for every natural mechanism.
\end{lem}
\begin{proof}
Couple the executions of $\cM$ and $\cM'$ against $\cA_{\naq}$ on the same sample, biases, query values, and internal randomness.  The two runs coincide up to (and excluding) the first round $\ell_0$ at which empirical accuracy is violated.  If that never happens, the runs are identical, and the conclusion for $\cM'$ is the conclusion for $\cM$.  Otherwise, in round $\ell_0$ itself we have $|a^{\ell_0} - q^{\ell_0}(X)| > 3\eps$, which implies item 1 by the triangle inequality.
\end{proof}
For the remainder of this subsection we therefore fix a natural mechanism satisfying $|a^j - q^j(X)| \leq 3\eps$ for every $j$, with probability $1$.

Throughout the proof, $|A^{j}\cap X|$ denotes the number of accused sample points at the start of round $j$.

Our analysis will be conducted on the intersection of the following three key events:

  \vspace{-5pt}
  \begin{enumerate}
  \item (``Few accusations''): Let $F$ be the event that the
    set of accused items outside of the sample is small: $|A^{k+1}\setminus X|
    \leq \eps N/8$. Since the $A^j$ are nested, event $F$ implies the same condition for all $j$ in $[k]$.  The probability of $F$ is bounded in Claim~\ref{clm:fewaccused} below.
    
    \vspace{-2pt}
  \item (``Representative sample values''): Let $G$ be the event that at
    every round $j \in [k]$, $\bigl|q^j(X) - (1-\tfrac{|A^{j} \cap X|}{n})\,p^j\bigr| \leq \eps/4$.
    By a standard Chernoff bound (and a union bound over rounds) we have
    $\pr{}{\neg G} \leq 2k\, e^{-n\eps^2/32}$.

    \vspace{-2pt}
  \item (``Representative queries''): Let $H$ be the event that
    $|\tilde q^j(\cP) - p^j|\leq \eps$ for all rounds $j\in [k]$---that is, each
    query's population average is close to the corresponding sampling bias $p^j$.
    By a standard Chernoff bound (and a union bound over rounds) we have
    $\pr{}{\neg H} \leq 2k\, e^{-\eps^2 N/3}$.
  \end{enumerate}

The following claim shows that the event $F$ occurs with high probability: no matter how the mechanism answers, very few of the items outside the sample are ever accused.
\begin{clm}[Few Accusations]\label{clm:fewaccused}
$\Pr(|A^{k+1}\setminus X|
  \leq \eps N/8) \geq 1 - e^{-\Omega(n)}$.
\end{clm}

\begin{proof}

Fix the biases $p^1,...,p^k$ as well as all the information
visible to the mechanism (the query values $\{q_i^j: i\in X, j\in
[k]\}$, as well as the answers $a^1,...,a^k$). We prove that the
probability of this event is high conditioned on any setting of these
variables. 

The main observation is that, once we condition on
the biases $p^j$, the
query values at $\{\tilde q_i^j: i\notin X, j\in[k]\}$ are independent with
$\tilde q_i^j\sim \Ber(p^j)$. This is true because $\cM$ is a natural
algorithm (so it sees only the query values for points in $X$) and,
more subtly, because the analyst's decisions about how to sample the
$p^j$'s, and which points in $X$ to include in the sets $A^j$, are
independent of the query values outside of $X$. By the principle of
deferred decisions, we may thus think of
the query values $\{\tilde q_i^j: i\notin X, j\in[k]\}$ as selected after the
interaction with the mechanism is complete.

Fix $i\notin X$. For every $j \in [k]$ and $i\notin X$, we have
$$
\ex{}{\tilde z_i^j} = \ex{}{\trunc{10\eps}(a^j - p^j) \cdot (\tilde q_i^j - p^j)} = \ex{}{\trunc{10\eps} (a^j - p^j)} \cdot \ex{}{\tilde q_i^j - p^j} = 0.
$$
By linearity of expectation, we also have 
$
\ex{}{\tilde z_i} = \ex{}{ \sum_{j = 1}^{k} \tilde z_i^j } = 0.
$

Next, note that $|z_i^j|\leq
10\eps$, since $\trunc{10\eps} (a^j - p^j) \in [-10\eps,10\eps]$ and
$q_i^j - p^j \in [-1,1]$. 
The terms $z_i^j$ are not independent, since if the absolute value of a partial sum
$\sum_{j=1}^\ell z_i^j$ ever exceeds $\tau-10\eps$, then subsequent values
$z_i^j$ for $j>\ell$
will be set to 0.  However, we may consider a related sequence given
by sums of the terms $\tilde
z _i^j = \trunc{10\eps}(a^j - p^j) \cdot (\tilde q_i^j - p^j)$ (the difference from $z_i^j$ is that we use values  $\tilde q_i^j \sim \Ber(p^j)$
regardless of whether item $i$ is in $A^j$). Once we have conditioned on the biases
and mechanism's outputs,
$\sum_{j=1}^k \tilde z_i^j$ is a sum of bounded independent random
variables.
By Hoeffding's Inequality, the sum is bounded by $O(\eps\sqrt{k\log(1/\eps)})$  with high probability, for every $i \not\in X$
$$
\pr{}{\left|\sum_{j=1}^{k} \tilde z_i^{j} \right| >
  \tfrac{10}{3}\eps \sqrt{18 k \ln\left(\frac{96}{\eps}\right)}} \leq
\frac{\eps}{48}\, .
$$
By Etemadi's Inequality,   a related bound holds uniformly
over all the intermediate sums:
$$
\forall i \not\in X~~~\pr{}{\exists \ell \in
  [k]~:~\left|\sum_{j=1}^{\ell} \tilde z_i^{j} \right| > \underbrace{10\eps
  \sqrt{18 k \ln\left(\frac{96}{\eps}\right)}}_{\tau-10\eps}} \leq 3 \cdot
\pr{}{\left|\sum_{j=1}^{k} \tilde z_i^{j} \right| > \tfrac{10}{3}\eps \sqrt{18 k \ln\left(\frac{96}{\eps}
\right)}} \leq \frac{\eps}{16}
$$
Finally, notice that by construction, the real scores $z_i^j$ are all
set to 0 when an item is added to $A^j$, so the sets $A^j$ are nested
($A^j \subseteq A^{j+1}$), and a bound on partial sums of the $\tilde
z_i^j$ applies equally well to the partial sums of the $z_i^j$. Thus, 
$$
\forall i \not\in X~\pr{}{\exists \ell \in
  [k] : \left|\sum_{j=1}^{\ell} z_i^{j} \right| > \tau-10\eps} 
\leq \frac{\eps}{16}
$$

Now, the scores $z_i$ are independent across items
(again, because we have fixed the biases $p^j$ and the mechanism's outputs). We can bound the
probability that more than $\frac{\eps N}{8}$ elements $i$ are
accused over the course of the algorithm using Chernoff's bound:
$$
\pr{}{| A^{k+1} \setminus X | > \tfrac{\eps}{8}N } \leq e^{-\eps N/64} \leq e^{-\Omega(n)}
$$
The claim now follows by averaging over all of the choices we fixed.
\end{proof}

  The proof of Claim~\ref{clm:natural_guilty} now rests on three sub-claims.  Sub-Claim~\ref{subclm:trunc} shows that, on the event $G$, the truncation in the algorithm is inactive at every round in which few sample points are accused, so the analyst's scores coincide with the untruncated ones.  Sub-Claim~\ref{subclm:gated} combines the reduction with the expectation bound of Claim~\ref{clm:bigscores}: up to the stopping time, every round contributes a constant conditional drift.  Sub-Claim~\ref{subclm:conc} then applies Azuma's inequality to conclude that the stopped sum of scores is large with high probability.

  \begin{subclm}\label{subclm:trunc}
    At every round $j$ that begins with $|A^{j} \cap X| < 6\eps n$, conditioned on
    $G$ we have $|a^j - p^j| \leq \tfrac{37}{4}\eps < 10\eps$.  In particular,
    at such rounds the truncation is inactive and $z^j_i = \hat z^j_i$ for all $i$.
  \end{subclm}
\vspace{-8pt}
  \begin{proof}
    Since the mechanism is empirically accurate, $|a^j - q^j(X)| \leq 3\eps$.
    By $G$, $|q^j(X) - (1-\tfrac{|A^{j} \cap X|}{n})p^j| \leq \eps/4$, and $\tfrac{|A^{j} \cap X|}{n} p^j <
    6\eps$.  Hence $|a^j - p^j| \leq 3\eps + \eps/4 + 6\eps = \tfrac{37}{4}\eps$.
  \end{proof}

  We track the following stopped sum $D$ of the in-sample untruncated scores:
\begin{equation}\label{eq:D}
\begin{aligned}
D \;:=\; \sum_{j=1}^{T} D_j \;&:=\; \sum_{j=1}^{T}\;\sum_{i \in X} \hat z^j_i, \\
  T \;&:=\; \min\Bigl(k,\ \min\bigl\{j : |A^{j} \cap X| \geq 6\eps n \text{ \; or \; } |A^j \setminus X| > \tfrac{\eps N}{8}\bigr\} - 1\Bigr),
\end{aligned}
\end{equation}
where the stopping
time $T$ halts the process before the first round that begins with a noticeable accusation fraction (in-sample or out-of-sample).

Note that whether $j \leq T$ is determined by the history through round
$j-1$ (by nestedness, it is equivalent to $|A^{j} \cap X| < 6\eps n$ and $|A^{j} \setminus X| \leq \tfrac{\eps N}{8}$, and $A^j$ is determined by the first $j-1$ rounds), so the hypothesis of the following sub-claim is
well-formed.
\begin{subclm}\label{subclm:gated}
For every round $j$ and every history through round $j-1$ under which $j \leq T$,
$$\ex{}{D_j \,\Big|\, \mathrm{history}} \;\geq\; \tfrac{1}{13}.$$
\end{subclm}
\vspace{-8pt}
\begin{proof}
By Claim~\ref{clm:bigscores}, it suffices to bound $\ex{}{|a^j - \bar q^j| \mid \text{history}}$.  We compute:
\begin{align*}
\ex{}{\,\bigl|a^j - \bar q^j\bigr| \,\Big|\, \text{history}}
\;&\overset{(1)}{\leq}\; 3\eps \;+\; \ex{}{\,\bigl|q^j(X) - \bar q^j\bigr| \,\Big|\, \text{history}} \\
\;&\overset{(2)}{\leq}\; 3\eps \;+\; \ex{}{\,\Bigl|q^j(X) - \Bigl(1 - \tfrac{|A^j \cap X|}{n}\Bigr)p^j\Bigr| \,\Big|\, \text{history}} \\
&\qquad\;\; +\; \ex{}{\,\tfrac{|A^j \cap X|}{n}\, p^j \,\Big|\, \text{history}} \;+\; \ex{}{\,\bigl|p^j - \bar q^j\bigr| \,\Big|\, \text{history}} \\
\;&\overset{(3)}{\leq}\; 3\eps + \tfrac{1}{2\sqrt{n}} + 6\eps + \tfrac{1}{\sqrt{2N}}
\;\overset{(4)}{\leq}\; 10\eps.
\end{align*}
Transition (1) uses empirical accuracy: $|a^j - q^j(X)| \leq 3\eps$ pointwise.
Transition (2) is the triangle inequality through $(1 - \tfrac{|A^j \cap X|}{n})p^j$ and $p^j$.
Transition (3): since $j \leq T$, we have $|A^{j} \cap X| < 6\eps n$ (so the middle term is at most $6\eps$, as $p^j \leq 1$) and $N - |A^j| \geq N/2$; conditioned on the history and on $p^j$, the round-$j$ values are fresh, so the two remaining terms --- sampling deviations of averages of $n$ and of $N - |A^j| \geq N/2$ Bernoulli draws --- have conditional expectations at most $\tfrac{1}{2\sqrt{n}}$ and $\tfrac{1}{\sqrt{2N}}$, respectively.
Transition (4) uses $\eps\sqrt{n} \geq \sqrt{C \log n} \geq 2$.
Altogether, by Claim~\ref{clm:bigscores}, $\ex{}{D_j \mid \text{history}} \geq \tfrac{1}{12} - 10\eps \geq \tfrac{1}{13}$ whenever $\eps$ is small enough.
\end{proof}

\vspace{-5pt}
  Let $S^j := \sum\limits_{i \in X\setminus A^j}(q^j_i - p^j)$ denote the round-$j$
in-sample sum, so that $\sum\limits_{i \in X}\hat z^j_i = (a^j - p^j)\,S^j$, and let
$B := \sqrt{8n\log(400nk)}$.

\begin{subclm}\label{subclm:conc}
There is an event $I$ with $\pr{}{I} \geq \tfrac{199}{200}$ on which $|S^j|
\leq B$ for every round $j$, and such that, on $G \cap I$,
$$D \;\geq\; \tfrac{T}{13} - \tfrac{k}{48}.$$
\end{subclm}
\vspace{-8pt}
\begin{proof}
Let $\bar D$ be defined as $D$, but with each summand $(a^j - p^j)S^j$
replaced by its truncation to $[-10\eps B,\, 10\eps B]$:
$$\bar D \;:=\; \sum_{j=1}^{T} \bar D_j \;:=\; \sum_{j=1}^{T}\trunc{10\eps B}\bigl((a^j - p^j)\,S^j\bigr).$$

For $j \in [k]$ let $X_j := \bigl(\tfrac{1}{13} - \bar D_j\bigr)\cdot\mathbf{1}\{j \leq T\}$,
so that $\sum_{j=1}^{k} X_j = \tfrac{T}{13} - \bar D$ and $|X_j| \leq 10\eps B + 1$.
We will apply the Azuma--Hoeffding inequality to the $X_j$'s, for which we must
bound their conditional expectations.

Fix $j$ and condition on the history through round $j-1$ (and hence on
$X_1, \dots, X_{j-1}$); we may assume $j \leq T$, as otherwise $X_j = 0$.  By
Sub-Claim~\ref{subclm:gated}, $\ex{}{\tfrac{1}{13} - D_j \mid
\text{history}} \leq 0$, so it suffices to bound $\ex{}{D_j - \bar D_j \mid
\text{history}}$, i.e., the effect of the truncation.  Conditioned additionally on $p^j$, the sum $S^j$ has at
most $n$ independent, centered, $[-1,1]$-bounded terms, so $\pr{}{|S^j| > B}
\leq 2e^{-2B^2/n} = 2(400nk)^{-16}$ by Hoeffding's inequality.  Moreover, if
$G$ holds then Sub-Claim~\ref{subclm:trunc} gives $|a^j - p^j| < 10\eps$, so
the truncation can be active only if $|S^j| > B$ or $G$ fails at round $j$,
the latter having conditional probability at most $2e^{-n\eps^2/32}$.  Since
$|D_j - \bar D_j| \leq |S^j| \leq n$, it follows that
$\ex{}{D_j - \bar D_j \mid \text{history}} \leq n\bigl(2(400nk)^{-16} +
2e^{-n\eps^2/32}\bigr) \leq 2(400k)^{-7}$, using $\eps^2 n \geq C\log n$ with
$C$ large, and $k = \mathrm{poly}(n)$.
Altogether, $\ex{}{X_j \mid X_1, \dots, X_{j-1}} \leq 2(400k)^{-7}$ for every
$j$.

Hence, by the Azuma--Hoeffding inequality, with probability at least $1 -
\tfrac{1}{400}$,
$$\bar D \;\geq\; \tfrac{T}{13} \;-\; 2k\,(400k)^{-7} \;-\;
\underbrace{(10\eps B + 1)\sqrt{2k\ln 400}}_{\leq\, k/50} \;\geq\;
\tfrac{T}{13} - \tfrac{k}{48},$$
where the bound $k/50$ follows from the choice of $k$ and $\eps^2 n \geq
C\log n$ (note $\log(400nk) \leq 4\log n$ for large $n$); call this the Azuma
event.  Finally, let $I$ be the intersection of the Azuma event with
$\{\forall j : |S^j| \leq B\}$; by a union bound, $\pr{}{I} \geq 1 -
\tfrac{1}{400} - 2k(400k)^{-16} \geq \tfrac{199}{200}$, and on $G \cap I$ the
truncation is never active, so $D = \bar D \geq \tfrac{T}{13} -
\tfrac{k}{48}$.
\end{proof}

\paragraph{Step 3: Putting it together.}
\begin{proof}[Proof of Claim~\ref{clm:natural_guilty}]
  We work on the event $F \cap G \cap H \cap I$.  Suppose first that
  $T \geq k - 20$.  Every round $j \leq T$ begins with $|A^{j} \cap X| < 6\eps n$, so
  by Sub-Claim~\ref{subclm:trunc} the truncation is inactive there and $z^j_i =
  \hat z^j_i$; on each of the at most $20$ remaining rounds, the total
  in-sample score increment is at most $10\eps|S^j| \leq 10\eps B$ in absolute
  value on $I$.  By Sub-Claim~\ref{subclm:conc},
  $$\sum_{i \in X} z_i \;=\; \sum_{j \leq T}\sum_{i \in X} \hat z^j_i \;+\;
  \sum_{j > T}\sum_{i \in X} z^j_i \;\geq\; D - 200\,\eps B \;\geq\;
  \tfrac{T}{13} - \tfrac{k}{48} - 200\,\eps B \;\geq\; \tfrac{k}{18}$$
  for sufficiently large $n$.  So if $T \geq k - 20$ then item 2 of
  Claim~\ref{clm:natural_guilty} holds with $c = 1/18$.

  Otherwise $T < k - 20$: round $T+1$ begins with
  $|A^{T+1} \cap X| \geq 6\eps n$ (on $F$, the second condition in the definition of $T$ never binds), and $|A^{j} \cap X|$ is non-decreasing in $j$.  The
  biases $p^{T+1},\dots,p^{k}$ are i.i.d.\ uniform and independent of the
  history through round $T$, so, except with probability $(3/4)^{20} \leq
  \tfrac{1}{300}$, some round $j^* > T$ has $p^{j^*} \geq \tfrac34$.  At that
  round, on $F \cap G \cap H$,
  \begin{align*}
  q^{j^*}(\cP) - q^{j^*}(X)
  \;&\overset{(1)}{\geq}\; \Bigl(\tilde q^{j^*}(\cP) - \tfrac{|A^{j^*}|}{N}\Bigr) - q^{j^*}(X) \\
  \;&\overset{(2)}{\geq}\; \Bigl(p^{j^*} - \eps - \tfrac{|A^{j^*}|}{N}\Bigr) - q^{j^*}(X) \\
  \;&\overset{(3)}{\geq}\; \Bigl(p^{j^*} - \eps - \tfrac{\eps}{4}\Bigr)
  - \Bigl((1-\tfrac{|A^{j^*} \cap X|}{n})\,p^{j^*} + \tfrac{\eps}{4}\Bigr)
  \;=\; \tfrac{|A^{j^*} \cap X|}{n}\, p^{j^*} - \tfrac{3\eps}{2} \\
  \;&\overset{(4)}{\geq}\; 6\eps\cdot\tfrac34 - \tfrac{3\eps}{2}
  \;=\; 3\eps \;>\; \eps.
  \end{align*}
  Transition (1) accounts for the zeroing of the accused items, which moves
  the population average by at most $\tfrac{|A^{j^*}|}{N}$; transition (2) is
  event $H$; transition (3) uses $|A^{j^*}| \leq |A^{j^*} \setminus X| +
  |A^{j^*} \cap X| \leq \tfrac{\eps N}{8} + n = \tfrac{\eps N}{4}$ (by $F$
  and the nestedness of the accusation sets) together with event $G$ for
  $q^{j^*}(X)$; transition (4) uses $|A^{j^*} \cap X| \geq |A^{T+1} \cap X| \geq
  6\eps n$ and $p^{j^*} \geq \tfrac34$.  So item 1 holds, witnessed by a post
  hoc generalization violation of the tracing query itself.  By a union bound,
  $$\pr{}{\neg(F \cap G \cap H \cap I)} + \bigl(\tfrac34\bigr)^{20} \;\leq\;
  e^{-\Omega(n)} + 2k\,e^{-n\eps^2/32} + 2k\,e^{-\eps^2 N/3} + \tfrac{1}{200}
  + \tfrac{1}{300} \;\leq\; \tfrac{1}{50}$$
  for sufficiently large $n$ with $\eps^2 n \geq C\log n$.  Thus item 1 or
  item 2 holds with probability at least $49/50$, and by
  Lemma~\ref{lem:patch} the same holds for every natural mechanism.  This
  completes the proof of Claim~\ref{clm:natural_guilty}.
\end{proof}
\vspace{-5pt}\section{Lower Bounds for General Algorithms}\label{sec:lifting}
In this section we show how to ``lift'' our lower bounds for natural algorithms to arbitrary algorithms using techniques from~\cite{HardtU14}.
Throughout this section, the source adversary $\cA_{\naq}$ chooses the uniform distribution on $[N]$.

\subsection{Information-Theoretic Lower Bounds via Random Masks}

We prove information-theoretic lower bounds by constructing the following transformation from an adversary that defeats all natural algorithms to an adversary (for a much larger universe) that defeats all algorithms.  The main idea of the reduction is to use random masks\footnote{In cryptographic terminology, a \emph{one-time pad}.} to hide information about the evaluation of the queries at points outside of the dataset.  Since the algorithm does not obtain any information about the evaluation of the queries on points outside of its dataset, it is effectively forced to behave like a natural algorithm.

\begin{algorithm}[ht]
\caption{$\cA_{\all}$}
Parameters: sample size $n$, universe size $N = \frac{8n}{\eps}$, number of queries $k$, target accuracy $\eps$.\\
Oracle: an adversary $\cA_{\naq}$ for natural algorithms with sample size $n$, universe size $N$, number of queries $k$, target accuracy $\eps$. \\ \vspace{3pt}

Let $\cX = \{(i,y)\}_{i \in [N], y \in \pmo^{k}}$\\
\For{$i \in [N]$}{
Choose $\key_i = (\key_i^1,\dots,\key_i^k) \sim \sU(\pmo^{k})$
}
Let $\cP$ be the uniform distribution over pairs $(i,\key_i)$ for $i \in [N]$

\For{$j \in [k]$}{
Receive the query $\hat{q}^j \from [N] \to [\pm 1]$ from $\cA_{\naq}$\\
Form the query $q^j(i,y) = y^j \oplus m_i^j \oplus \hat{q}^j(i)$ (NB: $q^j(i,m_i) = \hat{q}^j(i)$)\\
Send the query $q^j$ to $\cM$ and receive the answer $a^j$\\
Send the answer $a^j$ to $\cA_{\naq}$
}
\end{algorithm}

We now prove the following lemma, which states that if no natural algorithm can be post hoc generalizing and accurate against $\cA_{\naq}$, then no algorithm of any kind can be post hoc generalizing and accurate against $\cA_{\all}$.

\begin{lem} \label{lem:itreduction}
For every algorithm $\cM$, and every adversary $\cA_{\naq}$ for natural algorithms given as an oracle to $\cA_{\all}$, the adversary $\cA_{\all}$ satisfies
\begin{align*}
&\pr{\AQ{[N] \times \pmo^{k},n,k}{\cM}{\cA_{\all}}}{\exists j \in [k]~~\left| q^j(X) - q^j(\cP) \right| > \eps \bigvee \left| a^j - q^j(\cP) \right| > \eps  } \\
\geq{} &\min_{\textrm{natural $\cM$}}\left(\pr{\NAQ{[N],n,k}{\cM}{\cA_{\naq}}}{\exists j \in [k]~~\left| q^j(X) - q^j(\cP) \right| > \eps \bigvee \left| a^j - q^j(\cP) \right| > \eps }\right)
\end{align*}
\end{lem}

The following corollary is immediate by combining Theorem~\ref{thm:naturallb} with Lemma~\ref{lem:itreduction}.
\begin{coro} There is an adversary $\cA_{\all}$ and there are absolute constants $c_0, C > 0$ such that the following holds for every algorithm $\cM$ answering $k$ queries using $n$ samples, and for every $\sqrt{\frac{C\log n}{n}} \leq \eps \leq c_0$.  Let $N = 8n/\eps$.  If
$$
\pr{\AQ{[N] \times \pmo^{k},n,k}{\cM}{\cA_{\all}}}{\exists j \in [k]~~\left| q^j(X) - q^j(\cP) \right| > \eps \bigvee \left| a^j - q^j(\cP) \right| > \eps } \leq .01
$$
then $n = \Omega\left(\frac{\sqrt{k}}{\eps^2\sqrt{\log(1/\eps)}}\right)$.
\end{coro}

We now return to proving Lemma~\ref{lem:itreduction}.
\begin{proof}[Proof of Lemma~\ref{lem:itreduction}]
Fix any algorithm $\cM$.  We claim that there is an algorithm $\cM_{\naq}$ such that 
\begin{align}
\pr{\NAQ{[N],n,k}{\cM_{\naq}}{\cA_{\naq}}}{\exists j \in [k]~~\left| q^j(X) - q^j(\cP) \right| > \eps \bigvee \left| a^j - q^j(\cP) \right| > \eps  }& \notag \\
={} \pr{\AQ{[N] \times \pmo^{k},n,k}{\cM}{\cA_{\all}}}{\exists j \in [k]~~\left| q^j(X) - q^j(\cP) \right| > \eps \bigvee \left| a^j - q^j(\cP) \right| > \eps  }&  \label{eq:itreduction}
\end{align}
We construct the algorithm as follows
\begin{algorithm}[ht]
\caption{$\cM_{\naq}$}
Input: a sample $\hat{X} = (\hat{x}_1,\dots,\hat{x}_n) \in [N]^n$\\
Oracle: an algorithm $\cM$ \\ \vspace{3pt}

\For{$i \in [N]$}{
Choose $\key_i = (\key_i^1,\dots,\key_i^k) \sim \sU(\pmo^{k})$
}
\For{$i \in [n]$}{
Let $x_i = (\hat{x}_i, \key_{\hat{x}_i})$
}
Let $X = (x_1,\dots,x_n)$ \\ \vspace{3pt}

\For{$j \in [k]$}{
Receive the (partial) query $\hat{q}^j = (\hat{q}^j_{\hat{x}_1},\dots,\hat{q}^j_{\hat{x}_n})$ from the adversary \\
Form the query $$q^j(i,y) = \left\{ \begin{array}{cl} y^j \oplus m_i^j \oplus \hat{q}^j_i & i \in \hat{X} \\ y^j \oplus m_i^j \oplus 0& i \notin \hat{X} \end{array} \right.$$ \\
Run $\cM(X)$ on the query $q^j$ and receive the answer $a^j$\\
Return the answer $\hat{a}^j = a^j$
}
\end{algorithm}

Observe that, by construction, if $\hat{x}_1,\dots,\hat{x}_n$ are independent samples from the uniform distribution over $[N]$, then $x_1,\dots,x_n$ are independent samples from the uniform distribution over pairs $(i, m_i)$ for uniformly random $m_1,\dots,m_N \in \pmo^{k}$.  Thus, the distribution of $X$ is the same as the distribution constructed by $\cA_{\all}$.  

By construction, the distribution of the evaluations of the queries $q^1,\dots,q^k$ on points of the form $(i,y)$ for $i \in \hat{X}$ is identical to the distribution of those values in $\cA_{\all}$.  The only thing that requires verification is that the distribution of the evaluations of the queries $q^1,\dots,q^k$ on points of the form $(i,y)$ for $i \not\in \hat{X}$ are also identical to the distribution of those values in $\cA_{\all}$.  Now, since $m_1,\dots,m_N$ are jointly uniformly random and independent of $\hat{q}^1,\dots,\hat{q}^k$, the joint distribution of the values $m_i^j \oplus \hat{q}^j_i$ for $i \not\in \hat{X}$ and $j \in [k]$ is identical to that of $m_i^j \oplus 0$.  Since these values are enough to determine the evaluation of the queries $q^1,\dots,q^k$ on all points of the form $(i,y)$ for $i \not\in \hat{X}$, when conditioned on $X$, the distribution of the queries $q^1,\dots,q^k$ given to $\cM$ is identical to those constructed by $\cA_{\all}$.  Since the view of $\cM$ when interacting with $\cA_{\all}$ is identical to the view of $\cM_{\naq}$ when interacting with $\cA_{\naq}$, we must have~\eqref{eq:itreduction}.
\end{proof}

\subsection{Computational Lower Bounds via Pseudorandom Masks}

The drawback of the random masking procedure is that, since the length of the random mask is equal to the number of queries, and this random mask is the data given to the algorithm, the dimension of the data must be at least equal to the number of queries.  However, using cryptographic constructions, we can use a short seed to generate a large number of bits that appear random to any computationally efficient algorithm, thereby obtaining the same masking effect against computationally bounded adversaries, but with lower dimensional data.

First, we recall the definition of a pseudorandom generator.
\begin{defn} \label{def:prg} A function $\prg \from \pmo^{\ell} \to \pmo^{k}$ is an \emph{$(\ell,k,\gamma,T)$-pseudorandom generator} if for every algorithm $\cM$ running in time $T$ (i.e.~a circuit of size $T$)
$$
\left|\pr{}{\cM(G(\sU_{\ell}))} - \pr{}{\cM(\sU_{k})} \right| \leq \gamma
$$
where $\sU_{\ell}$ and $\sU_k$ denote the uniform distribution over $\pmo^{\ell}$ and $\pmo^{k}$, respectively.
\end{defn}

Definition~\ref{def:prg} is interesting when $\ell < k$.  When we say that $\cM$ runs in time $T$, we mean that $\cM$ is computed by a fixed program (e.g.~a RAM or a Boolean circuit) whose running time on all inputs $x \in \pmo^{k}$ is at most $T$.  A standard assumption in cryptography is that for arbitrarily small $c > 0$, arbitrarily large $C > 0$, and every sufficiently large $k \in \N$, there exists a $(k^{c},k,k^{-C},k^{C})$-pseudorandom generator.  This is the assumption we make in the introduction when we assume \emph{pseudorandom generators exist.}

\begin{rem}
Definition~\ref{def:prg} is \emph{non-uniform} in that it defines pseudorandom generators with a fixed output length $k$ and adversaries with a fixed input length $k$.  A more conventional definition would define a family of pseudorandom generators $\{G_k \from \pmo^{\ell(k)} \to \pmo^{k}\}$ and \emph{uniform} adversary (e.g.~a Turing machine) that takes an arbitrary input and has running time polynomial in the size of its input.  We could easily present our results using the latter definition, but we chose to use Definition~\ref{def:prg} since it is more concrete and corresponds more directly to our definition of post hoc generalization.
\end{rem}

A standard fact is that if $G$ is a pseudorandom generator, then it is pseudorandom even if we use it many times on independent random inputs.
\begin{fact} \label{fact:hybrid}
If $\prg \from \pmo^{\ell} \to \pmo^{k}$ is a $(\ell,k,\gamma,T)$-pseudorandom generator, then for every $m$, and every $\cM$ running in time $T$
$$
\left|\pr{}{\cM(\underbrace{G(\sU_{\ell}), \dots, G(\sU_{\ell})}_{\textrm{$m$ iid copies}})} - \pr{}{\cM(\underbrace{\sU_k,\dots,\sU_k}_{\textrm{$m$ iid copies}})} \right| \leq O(m\gamma)
$$
\end{fact}

We're now ready to state the reduction from computationally efficient algorithms that answer statistical queries to natural algorithms, which closely mimics the reduction $\cA_{\all}$ from arbitrary algorithms to natural algorithms.  In the reduction we use $G(s)^{j}$ to denote the $j$-th bit of the output of $G(s)$.

\begin{algorithm}[ht]
\caption{$\cA_{\comp}$}
Parameters: sample size $n$, universe size $N = \frac{8n}{\eps}$, number of queries $k$, target accuracy $\eps$, a pseudorandom generator $G \from \pmo^{\ell} \to \pmo^{k}$ \\
Oracle: an adversary $\cA_{\naq}$ for natural algorithms with sample size $n$, universe size $N$, number of queries $k$, target accuracy $\eps$. \\ \vspace{3pt}

Let $\cX = \{(i,z)\}_{i \in [N], z \in \pmo^{\ell}}$\\
\For{$i \in [N]$}{
Choose $\seed_i \sim \sU(\pmo^{\ell})$
}
Let $\cP$ be the uniform distribution over pairs $(i,\seed_i)$ for $i \in [N]$

\For{$j \in [k]$}{
Receive the query $\hat{q}^j \from [N] \to [\pm 1]$ from $\cA_{\naq}$\\
Form the query $q^j(i,z) = G(z)^j \oplus G(s_i)^j \oplus \hat{q}^j(i)$ (NB: $q^j(i,s_i) = \hat{q}^j(i)$)\\
Send the query $q^j$ to $\cM$ and receive the answer $a^j$\\
Send the answer $a^j$ to $\cA_{\naq}$
}
\end{algorithm}

We prove the following results about $\cA_{\comp}$.
\begin{lem} \label{lem:compreduction}
If there exists an $(\ell,k,\gamma,T)$-pseudorandom generator, then for every time $T$ algorithm $\cM$, and every adversary $\cA_{\naq}$ for natural algorithms given as an oracle to $\cA_{\comp}$, $\cA_{\comp}$ satisfies
\begin{align*}
&\pr{\AQ{[N] \times \pmo^{\ell},n,k}{\cM}{\cA_{\comp}}}{\exists j \in [k]~~\left| q^j(X) - q^j(\cP) \right| > \eps \bigvee \left| a^j - q^j(\cP) \right| > \eps  } \\
\geq{} &\min_{\textrm{natural $\cM$}}\left(\pr{\NAQ{[N],n,k}{\cM}{\cA_{\naq}}}{\exists j \in [k]~~\left| q^j(X) - q^j(\cP) \right| > \eps \bigvee \left| a^j - q^j(\cP) \right| > \eps }\right) - O(N\gamma)
\end{align*}
\end{lem}

The following corollary is immediate by combining Theorem~\ref{thm:naturallb} with Lemma~\ref{lem:compreduction}.
\begin{coro} If there exists an $(\ell,k,\gamma,T)$-pseudorandom generator, then there is an adversary $\cA_{\comp}$ and there are absolute constants $c_0, C > 0$ such that the following holds for every time $T$ algorithm $\cM$ answering $k$ queries using $n$ samples, and for every $\sqrt{\frac{C\log n}{n}} \leq \eps \leq c_0$.  Let $N = 8n/\eps$.  If
$$
\pr{\AQ{[N] \times \pmo^{\ell},n,k}{\cM}{\cA_{\comp}}}{\exists j \in [k]~~\left| q^j(X) - q^j(\cP) \right| > \eps \bigvee \left| a^j - q^j(\cP) \right| > \eps } \leq .01 - O(N\gamma)
$$
then $n = \Omega\left(\frac{\sqrt{k}}{\eps^2\sqrt{\log(1/\eps)}}\right)$.
\end{coro}

\begin{proof}[Proof of Lemma~\ref{lem:compreduction}]
The proof closely follows that of Lemma~\ref{lem:itreduction}, but replacing the uniform randomness of the masks $m_i$ with pseudorandomness of the masks $G(s_i)$.  That is, first we use the pseudorandom property of $G$ (Definition~\ref{def:prg} and Fact~\ref{fact:hybrid}) to argue that we can replace the strings $G(s_i)$ for the $N - n$ out-of-sample points $i \notin \hat X$ with uniformly random strings $m_i$, and if $\cM$ is computationally efficient this replacement changes the probability of the relevant failure event by at most $O(N\gamma)$.  At this point, since the masks are uniformly random, the view of the algorithm will be identical to its view when interacting with $\cA_{\all}$, so we can reuse the proof of Lemma~\ref{lem:itreduction} as is.
\end{proof}

\section{Post Hoc Generalization Does Not Compose} \label{sec:comp}

In this section we prove that post hoc generalization is not closed under composition.  We prove two results of this form---information-theoretic (Section~\ref{sec:infotheoreticcomp}) and computational (Section~\ref{sec:compcomp}).

\subsection{Information-Theoretic Composition} \label{sec:infotheoreticcomp}
\begin{theorem}\label{thm:noComposition}
For every $n\in\N$ and every $\alpha>0$ there is a collection of $\ell=O(\frac{1}{\alpha}\log n)$ algorithms $\cM_1,\dots,\cM_\ell \from (\zo^{5\log n})^n \to \cY$ such that
\begin{enumerate}
\item for every $i = 1,\dots,\ell$ and $\delta>0$, $\cM_i$ satisfies $(\eps,\delta)$-post hoc generalization for $\eps=O(\sqrt{\log(n/\delta)/n^{1-\alpha}})$, but
\item the composition $(\cM_1,\dots,\cM_\ell)$ is not $\left(2-\frac{3}{n^4},1-\frac{1}{2n^3}\right)$-post hoc generalizing.
\end{enumerate}
\end{theorem}

The result is based on an algorithm that we call \texttt{Encrypermute}.  Before proving Theorem~\ref{thm:noComposition}, we introduce \texttt{Encrypermute} and establish the main property that it satisfies.

\begin{algorithm}[ht]
\caption{$\texttt{Encrypermute}_k$\label{alg:mech_k}}
\KwIn{Parameter $k$, and a sample $X=(x_1,x_2,\dots,x_n)\in (\zo^{d})^n$ for $d = 5\log n$.}
\If{$X$ contains $n$ distinct elements}{
Let $\pi$ be the permutation that sorts $(x_1,\dots,x_k)$ and identify $\pi$ with $r\in\{0,1,\dots,k!-1\}$\\
Let $\alpha \in [0,1]$ be the largest number such that $k \geq n^{\alpha}$ and let $t \gets \alpha k /20$ (NB: $2^{dt} \leq k!$)\\
Identify $(x_{k+1},\dots,x_{k+t}) \in (\zo^{d})^{t}$ with a number $m \in \{0,1,\dots,k!-1\}$\\
\Return{$c = m+r \bmod{k!}$}
}
\Else{
\Return{a random number $c\in\{0,1,\dots,k!-1\}$}
}
\end{algorithm}

The key facts about \texttt{Encrypermute} are as follows.
\begin{claim}\label{clm:Encrypermute}
Let $\cD$ be any distribution over $(\zo^{d})^n$. Let $D \sim \cD$, let $X$ be a random permutation of $D$, and let $C\leftarrow\texttt{Encrypermute}(X)$. Then $D$ and $C$ are independent.
\end{claim}

\begin{proof}
Fix a possible dataset $\vec{y} = (y_1,\dots,y_n) \in (\zo^{d})^n$, and observe that if $\vec{y}$ does not contain $n$ distinct elements, then the output $C$ is uniform on $\{0,1,\dots,k!-1\}$. That is, for every such $\vec{y}$ and for every choice of $C=c$ we have
$
\pr{}{C=c \mid D=\vec{y}}=\frac{1}{k!}.
$

We now proceed assuming that $\vec{y}$ does contain $n$ distinct elements, in which case we will show that the output $C$ is still uniform. Fix a possible output $c \in \{0,1,\dots,k!-1\}$. Recall that $X$ is obtained by randomly permuting $\vec{y}$, and let $X_{k+1},\dots,X_{k+t}$ denote the elements in positions $k+1,\dots,k+t$ in $X$. We let $M$ denote the numerical value of $X_{k+1},\dots,X_{k+t}$. Let $R$ denote a random variables taking value $r$ in the algorithm, and observe that for every fixed value of $X_{k+1},\dots,X_{k+t}$, the order of the first $k$ elements in $X$ is still uniform, and hence, $R$ is uniform. Then,
\begin{align*}
\pr{}{C=c \mid D=\vec{y}} &= \pr{}{R=c-M \bmod{k!} \mid D=\vec{y}}\\
&= \sum_{\vec{x}=(x_{k+1},\dots,x_{k+t})} \pr{}{(X_{k+1},\dots,X_{k+t})=\vec{x} \mid D=\vec{y}]\cdot\Pr[R=c-M \bmod{k!} \mid \vec{y}, \vec{x} }\\
&= \sum_{\vec{x}=(x_{k+1},\dots,x_{k+t})}\Pr[(X_{k+1},\dots,X_{k+t})=\vec{x} \;|\; D=\vec{y}]\cdot\frac{1}{k!}=\frac{1}{k!}.
\end{align*}
Overall, for every choice of $\vec{y}$ (whether containing $n$ distinct elements or not) and for every choice of $C=c$ we have that $\pr{}{ C=c \mid D=\vec{y}}=1/k!$. Therefore, for every choice of $C=c$, $\pr{}{C=c} = \frac{1}{k!}$.
\begin{align*}
\pr{}{C=c}&=\sum_{\vec{y}}\pr{}{D=\vec{y}}\cdot\pr{}{C=c \mid D=\vec{y}}\\
&=\sum_{\vec{y}}\pr{}{D=\vec{y}}\cdot\frac{1}{k!}=\frac{1}{k!}.
\end{align*}
So, for every $\vec{y}$ and $c$ we have $\pr{}{C=c \mid D=\vec{y}}=\frac{1}{k!}=\pr{}{C=c}$, as desired.
\end{proof}

\begin{lemma}\label{lem:Encrypermute}For every $\delta > 0$ and $0 \leq k \leq n$,
$\texttt{Encrypermute}_k$ satisfies $(\eps,\delta)$-post hoc generalization for $\eps = \sqrt{2 \ln(2/\delta)/n}$.
\end{lemma} 

\begin{proof}
Let $d = 5 \log n$ and let $\cP$ be a distribution over $\{0,1\}^d$.  Let $\cA$ be any (possibly randomized) algorithm that outputs a function $q \from \zo^{d} \to [-1,1]$. Let $D \sim\cP^{\otimes n}$ be $n$ iid samples from $\cP$, let $X$ be a random permutation of $D$, let $C\leftarrow\texttt{Encrypermute}(X)$, and let $q \leftarrow \cA(C)$. By Claim~\ref{clm:Encrypermute}, $C$ and $D$ are independent, and so $q$ and $D$ are independent. Therefore, by the Hoeffding bound,
$$
\pr{}{\left|q(\cP)- q(D)\right|\geq\sqrt{\frac{2\ln(2/\delta)}{n}}}\leq\delta.
$$
As $X$ is a permutation of $D$, for every such $q$ we have $\ex{z \sim X}{q(z)} = \ex{z \sim D}{q(z)}$.  Therefore,
$$
\pr{}{\left|q(\cP)- q(X)\right|\geq\sqrt{\frac{2\ln(2/\delta)}{n}}}\leq\delta.
$$
showing that \texttt{Encrypermute} is indeed $(\eps,\delta)$-post hoc generalizing.
\end{proof}

\vspace{-5pt}
\begin{proof}[Proof of Theorem~\ref{thm:noComposition}]
Fix $\alpha \in (0,1)$, and let $\cM_1$ denote the mechanism that takes a database of size $n$ and outputs the first $n^\alpha$ elements of its sample. As $\cM_1$ outputs a sublinear portion of its input, it satisfies post hoc generalization with strong parameters. Specifically, by~\cite[Lemma 3.5]{CummingsLNRW16}, $\cM_1$ is $(\eps,\delta)$-post hoc generalizing for $$\eps=O\left(\sqrt{\frac{\log(n/\delta)}{n^{1-\alpha}}}\right)$$

Now consider composing $\cM_1$ with $\ell = O(\frac{1}{\alpha}\log n)$ copies of \texttt{Encrypermute}, with exponentially growing choices for the parameter $k_1,\dots,k_{\ell}$, where $$k_i = \left(1+\frac{\alpha}{20}\right)^{i-1}\cdot n^{\alpha}$$ By Lemma~\ref{lem:Encrypermute}, each of these mechanisms satisfies post hoc generalization for $\eps = O(\sqrt{\log(1/\delta)/n})$, so this composition satisfies the assumptions of the theorem.

Let $\cP$ be the uniform distribution over $\{0,1\}^d$, where $d=5\log n$, and let $X\sim\cP^{\otimes n}$.  By a standard analysis, $X$ contains $n$ distinct elements with probability at least $(1-\frac{1}{2n^3})$. Assuming that this is the case, we have that the first copy of \texttt{Encrypermute} outputs $c=m+r \bmod{k!}$, where $m$ encodes the rows of $X$ in positions $n^{\alpha}+1,\dots,(1+\frac{\alpha}{20})\cdot n^{\alpha}$, and where $r$ is a deterministic function of the first $n^\alpha$ rows of $X$. Hence, when composed with $\cM_1$, these two mechanisms reveal the first $(1+\frac{\alpha}{20})n^{\alpha}$ rows of $X$. By induction, the output of the composition of all the copies of \texttt{Encrypermute} with $\cM_1$ reveals all of $X$. Hence, from the output of this composition, we can define the predicate $q \from \zo^{d} \to \pmo$ that evaluates to 1 on every element of $X$, and to $-1$ otherwise. This predicate satisfies $q(X) = 1$ but $$q(\cP) \leq -1 + \frac{2n}{2^d} = -1 + \frac{2}{n^4}.$$
\end{proof}

\subsection{Computational Post Hoc Generalization} \label{sec:compcomp}

In this section we consider a computational variant of post hoc generalization, in which we only require that $\cM$ prevents a \emph{computationally efficient} algorithm from identifying a statistical query such that the sample and population give very different answers.  First we formally define \emph{computational post hoc generalization} (after also introducing a technical definition).

\begin{defn}
Let $\cP$ be a family of distributions over $\pmo^{d}$.  The family is \emph{$T$-sampleable} if there is a randomized algorithm $M$ taking no inputs (e.g.~a RAM or circuit) with running time $T$ such that the output distribution of $M$ is identical to $\cP$.
\end{defn}

\begin{defn}[Computational Post Hoc Generalization] 
An algorithm $\cM \from (\pmo^{d})^{n} \to \cY$ satisfies \emph{$(n,d,\eps,\delta,T)$-computational post hoc generalization} if for every $T$-sampleable distribution $\cP$ over $\pmo^{d}$ and algorithm $\cA$ running in time $T$ that outputs a bounded function $q \from \pmo^{d} \to [-1,1]$, if $X \sim \cP^{\otimes n}, y \sim \cM(X),$ and $\query \sim \cA(y)$, then
$$
\pr{X,y,q}{\left| \ex{z \sim \cP}{q(z)} - \ex{z \sim X}{q(z)} \right| > \eps} \leq \delta,
$$
where the probability is over the sampling of $X$ and any randomness of $\cM,\cA$.  The notation $z \sim X$ indicates that $z$ is a uniformly random element of the sample $X \in \cX^n$.
\end{defn}

\begin{theorem} \label{thm:noCompositionComp}
For every sufficiently small constant $\alpha>0$, and every sufficiently large $n \in \N$, assuming the existence of an $(n^{\alpha},n^{1.01},n^{-200},T^{100})$-pseudorandom generator, there are $\cM_1,\cM_2 \from (\pmo^{5 \log n})^{n} \to \cY$ such that 
\begin{enumerate}
\item for every $\delta \geq n^{-100}$, both $\cM_1$ and $\cM_2$ satisfy $(\eps,\delta,T)$-computational post hoc generalization for $\eps = O(\sqrt{\ln(1/\delta)/n^{.999}})$, but
\item The composition $(\cM_1,\cM_2)$ is not $\left(1-\frac{1}{n^3},1-\frac{1}{n^3},T\right)$-computational post hoc generalizing.
\end{enumerate}
\end{theorem}
We note that the parameters of the assumed pseudorandom generator are somewhat arbitrary, and are somewhat irrelevant since we are content to assume pseudorandom generators with arbitrarily-small-polynomial seed length secure against arbitrarily-large-polynomial time algorithms.  We chose these parameters for simplicity and concreteness.  The precise parameters we need will be clear from the construction.

The proof of Theorem~\ref{thm:noCompositionComp} is quite similar to that of Theorem~\ref{thm:noComposition}.  The main difference is that the computational assumption allows us to avoid the recursive construction involving $O(\log n)$ algorithms with a faster process involving just two algorithms.  Recall that in $\texttt{Encrypermute}_{k}$ we used the order of the first $k$ elements as a \emph{key} of length $\log_2(k!) = \Theta(k \log k)$.  We then used this random string as a way to \emph{encrypt} the next $\Omega(k)$ elements of the dataset.  The bottleneck is that in the information-theoretic setting we can only encrypt as many bits as the length of the key, and this meant we needed $\Omega(\log n)$ rounds to gradually encrypt the entire dataset.  However, using cryptographic assumptions, we can use a pseudorandom generator to stretch the $\Theta(k \log k)$ random bits into $(n-k)d$ random bits that appear truly random to any computationally bounded algorithm.  Then we can use this long key to encrypt the entire rest of the dataset.  A small technical issue is that pseudorandom generators are used to stretch uniformly random strings, whereas in \texttt{Encrypermute} we are stretching a random number $r \in \{0,1,\dots,k!-1\}$, and $k!$ is not a power of 2. This small gap can be handled using the following fact:

\begin{fact}\label{fact:SD}
Let $N\in\N$, and let $R$ be uniform on $[N]$. Then the $\frac{1}{2}\log N$ least significant bits in the binary representation of $R$ are almost uniform, with statistical distance at most $1/\sqrt{N}$.
\end{fact}

\begin{proof}[Proof Sketch]
The proof is based on the fact that (no matter how many of the most significant bits we ignore) the difference in count between different strings is at most 1. If we ignore half of the bits, the difference in occurrence is between $2^{\frac12\cdot\log N}$ and $2^{\frac12\cdot\log N}$+1, which introduces a statistical distance of at most $2^{-\frac12\cdot\log N} = 1/\sqrt{N}$ from uniform.
\end{proof}

\begin{algorithm}[ht]
\caption{$\texttt{PRG-Encrypermute}_k$\label{alg:mech_PRGEP}}
\KwIn{Parameter $k$, and a sample $X=(x_1,x_2,\dots,x_n)\in (\zo^{d})^n$ for $d = 5\log n$.} \vspace{5pt}
\texttt{Step (1):}\\
\If{$X$ contains $n$ distinct elements}{
Let $\pi$ be the permutation that sorts $(x_1,\dots,x_k)$ and identify $\pi$ with $p\in\{0,1,\dots,k!-1\}$}
\Else{
Let $p$ be a random number in $\{0,1,\dots,k!-1\}$
}\
Let $s$ denote the $\ell = \frac{k}{8} \log k$ least significant bits in the binary representation of $p$.\\\vspace{5pt}

\texttt{Step (2):}\\
Compute $r \gets G(s)$ for a pseudorandom generator $G \from \zo^{\ell} \to \zo^{d(n-k)}$\vspace{5pt}

\texttt{Step (3):}\\
Identify $x_{k+1},\dots,x_n \in (\zo^{d})^{n-k}$ with a binary string $m \in \zo^{d(n-k)}$\\
\Return{$c = m \oplus r$}
\end{algorithm}

Before analyzing the algorithm, we will introduce some notation.  For a sample $X$, let $X[k+1,n] = (x_{k+1},\dots,x_{n})$ denote the last $(n-k)$ elements of $X$.  We will break the algorithm \texttt{PRG-Encrypermute} into three algorithms, $\cM_{1},\cM_{2},\cM_{3}$ corresponding to the three steps.  The first algorithm $\cM_{1}$ takes a sample $X$ and outputs $s \in \zo^{\ell}$, the second algorithm $\cM_{2}$ takes $s \in \zo^{\ell}$ and outputs $G(s) \in \zo^{d(n-k)}$, and the third algorithm $\cM_{3}$ takes $X[k+1,n]$ and $r$ and outputs $c \in \zo^{d(n-k)}$.  With this notation, 
$$
\texttt{PRG-Encrypermute}_{k}(X) = \cM_{3}\left( X[k+1,n], \cM_{2}(\cM_{1}(X)) \right)
$$

\begin{clm}\label{obs:PRG1}
The algorithm that takes a database $X\in (\zo^d)^n$, samples $r\in\{0,1\}^{d(n-k)}$ uniformly at random and runs $\cM_3\left(X[k+1,n],r\right)$, is $(\eps,\delta)$-post hoc generalizing for every $\delta > 0$ and $\eps = \sqrt{2\ln(2/\delta)/n}$.
\end{clm}

\begin{proof}[Proof Sketch]
For every distribution $\cP$, the output of the algorithm is independent of the sample $X$, so the statement follows from the Hoeffding bound.
\end{proof}

\begin{claim}\label{claim:PRG2}
For every constant $\alpha > 0$, the algorithm that takes a sample $X\in (\zo^d)^n$, samples $s\in \zo^{\ell}$ uniformly at random and runs $\cM_3\left(X[k+1,n],\cM_2(s)\right)$ with parameter $k=n^\alpha$, is $(\eps,\delta)$-computationally post hoc generalizing for every $\delta \geq n^{-101}$ and $\eps = \sqrt{18\ln(8/\delta)/n}$.
\end{claim}

\begin{proof}
Assume towards contradiction that the statement is false. Let $\cP$ be efficiently samplable and $\cA$ be an efficient algorithm such that when we sample $X\sim\cP^{\otimes n}$, $s\in\{0,1\}^\ell$, $c\leftarrow\cM_3\left(X[k+1,n],\cM_2(s)\right)$, and $q\leftarrow\cA(c)$, we have
$$
\pr{}{\left| \ex{z \sim \cP}{q(z)} - \ex{z \sim X}{q(z)}\right| > \sqrt{18\ln(2/\delta)/n}}> 4\delta
$$
for some $\delta \in (0,1)$.

We now use $\cP$ and $\cA$ in order to construct the following distinguisher $\cB$ that contradicts our assumption that $\cM_{2}$ is a pseudorandom generator. 

\begin{algorithm}[ht]
\caption{Distinguisher $\cB$\label{alg:distinguisher}}
\KwIn{Parameter $\delta\in(0,1)$ and a string $r\in\{0,1\}^{d(n-k)}$} \vspace{5pt}
Sample two independent datasets $X,D \sim \cP^{\otimes n}$\\
Run $\cA\left( \cM_3\left(X[k+1,n],r\right)  \right)$ to obtain a query $q: \{0,1\}^d\rightarrow[-1,1]$\\
\Return{$1$ if $|q(X) - q(D)| > \sqrt{8 \ln (2/\delta)/n}$ and $0$ otherwise}
\end{algorithm}

First observe that in the algorithm above, $q$ and $D$ are independent, so by the Hoeffding bound, with probability at least $1-\delta$ we have 
\begin{equation}
|q(D)-q(\cP)|\leq\sqrt{\frac{2\ln(2/\delta)}{n}}.\label{eq:indSample}
\end{equation}

Now, if algorithm $\cB$ is executed with a uniform string $r$, then by Claim~\ref{obs:PRG1} with probability $1-\delta$ we also have $|q(X)-q(\cP)|\leq\sqrt{2\ln(2/\delta)/n}$. So, if $r$ is uniform, then the probability that $|q(X)-q(D)|>\sqrt{8\ln(2/\delta)/n}$ (i.e.~the probability that the algorithm outputs 1) is at most $2\delta$. 

On the other hand, by our assumption towards contradiction, when algorithm $\cB$ is executed with a string $r$ that is generated by applying $\cM_2$ on a uniformly random $s\in\{0,1\}^\ell$, then the probability that $|q(X)-q(\cP)|>\sqrt{18\ln(2/\delta)/n}$ is at least $4\delta$. Combined with Inequality~(\ref{eq:indSample}), we see that the probability of the algorithm outputting 1 in this case is at least $3\delta$. This contradicts our assumption on $\cM_2$ whenever $\delta \geq n^{-101}$.
\end{proof}

\begin{claim}\label{clm:smallSD}
Let $\cD$ be any distribution over $(\zo^d)^n$. Let $D\sim\cD$, let $X$ be a permutation of $D$, and let $s\leftarrow\cM_{1}(X)$. Also let $\tilde{s}$ be uniform on $\{0,1\}^{\ell}$, independent of $D$. Then for every parameter $k$,
$$
\SD\big( \;\left(D,X[k+1,n],s\right)\;\;,\;\; \left(D,X[k+1,n],\tilde{s}\right) \; \big) \leq k^{-k/8}.
$$
\end{claim}

\begin{proof}[Proof Sketch]
Consider $p\in\{0,1,2,\dots,k!-1\}$ defined in step~1 of $\texttt{PRG-Encrypermute}_{k}$. Similarly to Claim~\ref{clm:Encrypermute}, it holds that $p$ is uniform, and independent of $(D,X[k+1,n])$. Therefore $s$ (the string of length $\ell$ defined on step~2) is independent of $(D,X[k+1,n])$, since it is a function of $p$ only. We also have that $\tilde{s}$ is independent of $(D,X[k+1,n])$ by definition, and hence, the result follows immediately from Fact~\ref{fact:SD}.
\end{proof}

\begin{lemma}\label{lem:PRG-Encrypermute}
For every constant $\alpha>0$, the algorithm \texttt{PRG-Encrypermute} with parameter $k = n^\alpha$ is $(\eps,\delta)$-computationally post hoc generalizing for every $\delta \geq n^{-100}$ and $\eps = \sqrt{18\ln(16/\delta)/n}$.
\end{lemma}

\begin{proof}
Let $\delta\geq n^{-100}$ and $\eps=\sqrt{18\ln(16/\delta)/n}$.
Let $\cP$ be a sampleable distribution over $\{0,1\}^d$, let $D\sim\cP^{\otimes n}$, let $X$ be a permutation of $D$, and let $s\leftarrow\cM_{1}(X)$. Also let $\tilde{s}\in\{0,1\}^\ell$ be chosen uniformly at random. Now consider any algorithm $\cA$ that takes the output of $\texttt{PRG-Encrypermute}(X)$ and outputs a function $q$. Our goal is to show that the probability that $|q(\cP)-q(X)|>\eps$ is at most $\delta$. To that end, recall that we can restate $\cA\left(\texttt{PRG-Encrypermute}_{k}(X)\right)$ as
$$
\cA\big(\cM_{3}\left( X[k+1,n], \cM_{2}(\cM_{1}(X)) \right)\big)
$$
Therefore, the event that $|q(\cP)-q(X)|>\eps$ is a function of $(D,X[k+1,n],r)$, and the random coins of $\cA$. Hence, by Claim~\ref{clm:smallSD}, this event happens with roughly the same probability even when we replace $s$ with the uniform string $\tilde{s}$. In that case, by Claim~\ref{claim:PRG2} (applied with $\delta/2 \geq n^{-101}$), we know that this event happens with probability at most $\delta/2$.
\end{proof}

Theorem~\ref{thm:noCompositionComp} now follows from the above analysis by composing $\texttt{PRG-Encrypermute}_{k}$ with the algorithm that outputs the first $k$ elements of its input $X \in (\zo^{d})^n$.  The analysis is similar to that of Theorem~\ref{thm:noComposition}.

\addcontentsline{toc}{section}{References}
\bibliographystyle{plainnat}

\end{document}